\documentclass{article}

\usepackage{microtype}
\usepackage{graphicx}
\usepackage{enumitem}
\setlist{topsep=0pt, leftmargin=*}
\usepackage{subfigure}
\usepackage{booktabs} 

\usepackage{hyperref}

\usepackage[accepted]{icml2025}

\usepackage{amsmath}
\usepackage{amssymb}
\usepackage{mathtools}
\usepackage{amsthm}

\usepackage{siunitx}

\usepackage[capitalize,noabbrev]{cleveref}

\theoremstyle{plain}
\newtheorem{theorem}{Theorem}[section]
\newtheorem{proposition}[theorem]{Proposition}
\newtheorem{lemma}[theorem]{Lemma}

\theoremstyle{definition}
\newtheorem{definition}[theorem]{Definition}

\theoremstyle{remark}
\newtheorem{remark}[theorem]{Remark}

\usepackage[textsize=tiny]{todonotes}

\usepackage[T1]{fontenc}
\usepackage{hyperref} 
\begin{document}

\twocolumn[
\icmltitle{HoP: Homeomorphic Polar Learning for Hard Constrained Optimization}



\icmlsetsymbol{equal}{*}
\begin{icmlauthorlist}
\icmlauthor{Ke Deng}{equal,yyy}
\icmlauthor{Hanwen Zhang}{equal,yyy}
\icmlauthor{Jin Lu}{xx}
\icmlauthor{Haijian Sun}{yyy}
\end{icmlauthorlist}

\icmlaffiliation{yyy}{School of Electrical \& Computer Engineering University of Georgia, Athens, USA}
\icmlaffiliation{xx}{School of Computing, University of Georgia, Athens, USA}
\icmlcorrespondingauthor{Ke Deng}{ke.deng@uga.edu}
\icmlcorrespondingauthor{Hanwen Zhang}{hanwen.zhang@uga.edu}
\icmlcorrespondingauthor{Jin Lu}{jin.lu@uga.edu}
\icmlcorrespondingauthor{Haijian Sun}{hsun@uga.edu}

\vskip 0.3in
]
\printAffiliationsAndNotice{\icmlEqualContribution} 


\begin{abstract}
Constrained optimization demands highly efficient solvers which promotes the development of learn-to-optimize (L2O) approaches. As a data-driven method, L2O leverages neural networks to efficiently produce approximate solutions. However, a significant challenge remains in ensuring both optimality and feasibility of neural networks' output. To tackle this issue, we introduce Homeomorphic Polar Learning (HoP) to solve the star-convex hard-constrained optimization by embedding homeomorphic mapping in neural networks. The bijective structure enables end-to-end training without extra penalty or correction. For performance evaluation, we evaluate HoP's performance across a variety of synthetic optimization tasks and real-world applications in wireless communications. In all cases, HoP achieves solutions closer to the optimum than existing L2O methods while strictly maintaining feasibility.

\end{abstract}
\section{Introduction}
\label{Introduction}
\begin{figure}[ht]
\vspace{-1em}
\begin{center}
\centerline{\includegraphics[width=0.38\textwidth] {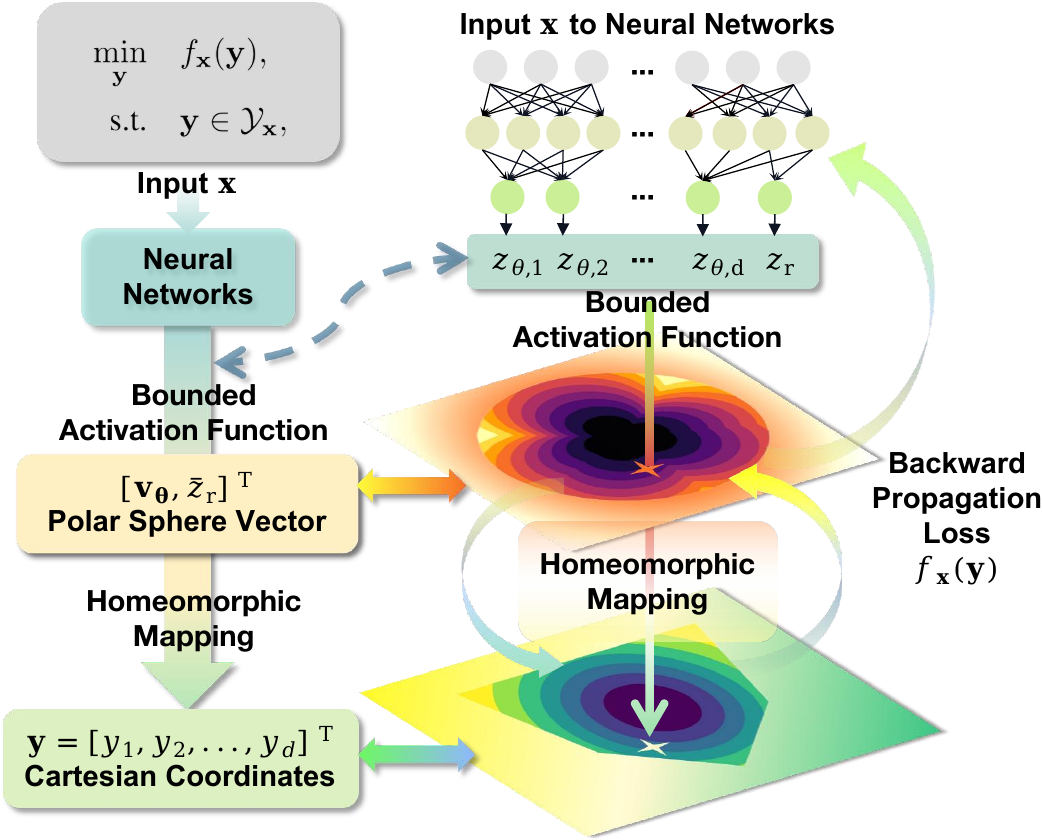}} 
\caption{HoP is structured as follows: the gray box defines the problem considered in this paper, where the problem variable is expressed as $\mathbf{y}$, where $\mathbf{x}$ denotes problem parameters, and the objective function, respectively. The parameters $\mathbf{x}$ are fed into a NN, which contains a bounded activation function and produces a polar sphere vector comprising the direction vector $\mathbf{v}_{{\theta}}$ and the length scale $\bar{z}_r$. Using a homeomorphic mapping, the polar sphere vector is then transformed into Cartesian coordinates while strictly adhering to the original constraints. The warm-colored space on the right side represents the polar space corresponding to the polar sphere vector, while the green space is the Euclidean space for Cartesian coordinates. The loss function $f_\mathbf{x}(\mathbf{y})$ which can be trained end-to-end without requiring additional penalties or corrections.}
\label{polar_projection_struc}
\end{center}
\vspace{-3em}
\end{figure}
Optimization takes a fundamental role in numerous scientific and engineering applications \cite{liu2024survey,abido2002optimal,rockafellar2013fundamental}. However, these algorithms often require significant computation time to solve complex problems, particularly when dealing with non-convex or high-dimensional cases. To address these limitations, the paradigm of learn-to-optimize (L2O) has emerged as a promising alternative \cite{hounie2024resilient,ding2024resilient}. As a learning-based scheme, L2O takes optimization parameters as the inputs of neural network (NN) which can efficiently obtain approximate solutions from the outputs. This data-driven approach enhances the speed and reduces the computational cost of achieving satisfactory solutions \cite{donti2021dc3,park2023self}.


Despite its advantages, L2O faces significant challenges when applied to constrained optimization. A major issue is the lack of guarantees for NN to ensure that solutions strictly remain within the feasible region defined by the constraints \cite{donti2021dc3}. Current works have attempted to address this limitation through various approaches, including supervised learning \cite{zamzam2020learning}, incorporating constrained violations into the loss function \cite{zhang2024constrained, xu2018semantic}, post-correction methods \cite{donti2021dc3}, implicit differentiation \cite{amos2017optnet}, and other techniques \cite{zhong2023neural,misra2022learning}. However, these methods often exhibit limitations in optimal solution searching and hard constraints violation control.

We propose the homeomorphic polar learning (HoP) to address the non-convex problem with star-convex hard constraints. As stated in Definition \ref{star_convex}, star-convexity is a weaker condition than convexity, which introduces additional challenges in handling hard constraints. HoP is a learning based framework inspired by principles of convexity and topological structure of constraints, designed to ensure both the feasibility and optimality of solutions. As illustrated in Fig.~\ref{polar_projection_struc}, the problem parameters are fed into NN, where the raw outputs are regulated by bounded functions. Subsequently, through the proposed homeomorphic mapping, the polar sphere vectors are mapped to Cartesian coordinates, strictly following to the original constraints. Furthermore, HoP is trained end-to-end with objective function as loss directly, which makes it adaptable to various applications. The key contributions are
\begin{itemize}
    \item \textbf{Novel formulation of hard-constrained optimization via polar coordinates and homeomorphic mapping}. HoP is the first L2O framework based on polar coordinates and homeomorphism to solve hard constrained problem. Our formulation extends the applicability of L2O methods to star-convex constraints, a more complex and less explored constraint type.
    \item \textbf{Reconnection strategies and dynamic learning rate adjustment and  for polar optimization}. To address challenges specific to polar coordinate optimization, such as radial stagnation and angular freezing, we propose a dynamic learning rate adjustment scheme and geometric reconnection strategies. 
    We provide rigorous theoretical analyses to validate the stability and efficiency of these solutions.
     \item \textbf{Superior experimental performance with zero violation.} 
    Through extensive ablation and comparative experiments, we validate the feasibility and optimality of our approach across a wide range of problems. Results consistently show that HoP outperforms both traditional and learning-based solvers in terms of constraint satisfaction and optimization efficiency.
\end{itemize}

\section{Related Work}
We present related works for constrained optimization problems using L2O. Broadly, researches in this area can be categorized into two distinct directions: soft and hard constrained L2O.

\subsection{Soft Constrained Optimization with L2O}
Soft constrained L2O emphasizes on enhancing computational efficiency on NN inference speed while tolerating a limited rate of constraint violations. Early research in this domain explored the use of supervised learning (SL) to directly solve optimization problems, where the optimal variable $\mathbf{y}^*$ is provided as labels by optimizer \cite{zamzam2020learning,guha2019machine}.  Another prominent direction involves incorporating constraint violations into the objective function by Karush-Kuhn-Tucker conditions \cite{donti2021dc3,zhang2024constrained, xu2018semantic}. In this methods, constraints are reformulated as penalty terms and integrated into the objective function as the loss for self-supervised learning (SSL). Subsequent advancements introduced alternative optimization based learning, where variables and multipliers are alternately optimized through dual NNs \cite{park2023self,kim2023self,nandwani2019primal}. More recent related researches include preventive learning, which incorporates pre-processing in learning to avoid violation \cite{zhao2023ensuring}. Additionally, resilience-based constraint relaxation methods dynamically adjust constraints throughout the learning process to balance feasibility and overall performance \cite{hounie2024resilient,ding2024resilient}.

\subsection{Hard Constrained Optimization with L2O}
Hard constrained optimization in L2O prioritizes strict adherence to constraints, even if it results in reduced optimality or slower computation speed. Traditional optimization methods often employ proximal optimization techniques to guarantee feasibility \cite{cristian2023end,min2024hard}. Early methods also used activation functions to enforce basic hard constraints \cite{sun2018learning}. Implicit differentiation became a popular approach for effectively handling equality constraints \cite{amos2017optnet,donti2021dc3,huang2021deepopf}. However, inequality constraints typically require additional correction steps, which can lead to suboptimal solutions \cite{donti2021dc3}. An alternative strategy proposed by \cite{li2023learning} utilized the geometric properties of linear constraints to ensure outputs within the feasible region, although this method is limited to linear constraints. Other studies such as \cite{misra2022learning,guha2019machine} focused on eliminating redundant constraints to improve inference speed instead of solving optimization problem by NNs directly. In certain physical applications, discrete mesh-based approaches restrict feasible solutions to predefined points on a mesh \cite{amos2017input,zhong2023neural,negiar2022learning}. While these methods strictly enforce feasibility, they often lack  flexibility in general scenarios.

\section{Methodology}
This section introduces the problem formulation, the proposed HoP framework, and associated theoretical analyses.
\subsection{Problem Formulation}

We consider the following optimization problem:
\begin{flalign}\label{problem_formulaion}
    \min_\mathbf{y} \quad f_\mathbf{x}(\mathbf{y}),
    \quad \text{s.t.} \quad  \mathbf{y} \in \mathcal{Y}_\mathbf{x}, 
\end{flalign}
where $f_\mathbf{x}: \mathbb{R}^n \to \mathbb{R}$ is the objective function, and $\mathcal{Y}_\mathbf{x}$ is defined as star-convex constraint set, both parameterized by $\mathbf{x}$. Star-convexity is a weaker condition than convexity, which exists in $\ell_p$-norm problem and compressed sensing applications \cite{yang2022towards,donoho2011compressed}. Star-convexity is defined as below:
\begin{definition} \label{star_convex}
Let $\mathcal{Y} \subset \mathbb{R}^n$ be a non-empty set. We define $\mathcal{Y}$ as star-convex if $\exists \mathbf{y}_0 \in \mathcal{Y}$ such that $\forall \mathbf{y} \in \mathcal{Y}$, the following holds:
\begin{flalign}
    \bigl\{\, \mathbf{y}_0 + t(\mathbf{y} - \mathbf{y}_0) \mid 0 \leq t \leq 1 \,\bigr\} \subseteq \mathcal{Y}.
\end{flalign}
\end{definition}


If this condition holds for $\forall\mathbf{y}_0 \in \mathcal{Y}$, $\mathcal{Y}$ becomes convex. Thus, star-convex sets generalize convex sets, accommodating constraint sets that are not globally convex while maintaining a degree of geometric regularity relative to specific points. In Section \ref{sec_hop}, we design HoP by the convexity in star-convex set.

\subsection{Homeomorphic Polar Learning}\label{sec_hop}

To demonstrate the proposed method, we first introduce the essential idea of HoP: homeomorphic mappings. These mappings leverage the mathematical properties of homeomorphisms, formally defined as follows:
\begin{definition}\label{Homeomorphism_def}
Let $ X = (S_X, \mathcal{T}_X) $ and $ Y = (S_Y, \mathcal{T}_Y) $ be two topological spaces, where: (1) $ S_X $ and $ S_Y $ are point sets;  (2) $ \mathcal{T}_X $ and $ \mathcal{T}_Y $ are topologies on $ S_X $ and $ S_Y $, respectively. Then function $\mathcal{H} \colon X \to Y$ is called a {homeomorphism} if and only if $ \mathcal{H} $ is a bijection and continuous with respect to the topologies $ \mathcal{T}_X $ and $ \mathcal{T}_Y $, while its inverse function $\mathcal{H}^{-1} \colon Y \to X $ exists and is also continuous.
\end{definition}

\paragraph{The 1-D Case}

To simplify the mechanism of homeomorphic mappings, we begin with a one-dimensional optimization problem with constraints $a < y < b$, where the feasible region $\mathcal{Y}_{\mathbf{x}}$ is a bounded interval $(a, b)$. The corresponding homeomorphic mapping $\mathcal{H}$ is defined as:
\begin{flalign}\label{1d_case_eq}
\mathcal{H}\colon \quad \hat{y} = a + \mathcal{B}(z) (b-a),
\end{flalign}
where $z$ is output from NN, $\mathcal{B}(z)$ is a bounded, smooth and monotonic function (e.g., Sigmoid function), mapping $z$ from $\mathbb{R}$ to $(0, 1)$. To guarantee the feasible outputs, we scale the output from $(0,1)$ to $(a,b)$ by Eq. (\ref{1d_case_eq}) which is considered as simple one-to-one homeomorphic mapping defined in Definition \ref{Homeomorphism_def}.

\begin{figure}[ht]
\begin{center}
\centerline{\includegraphics[width=0.8\columnwidth]{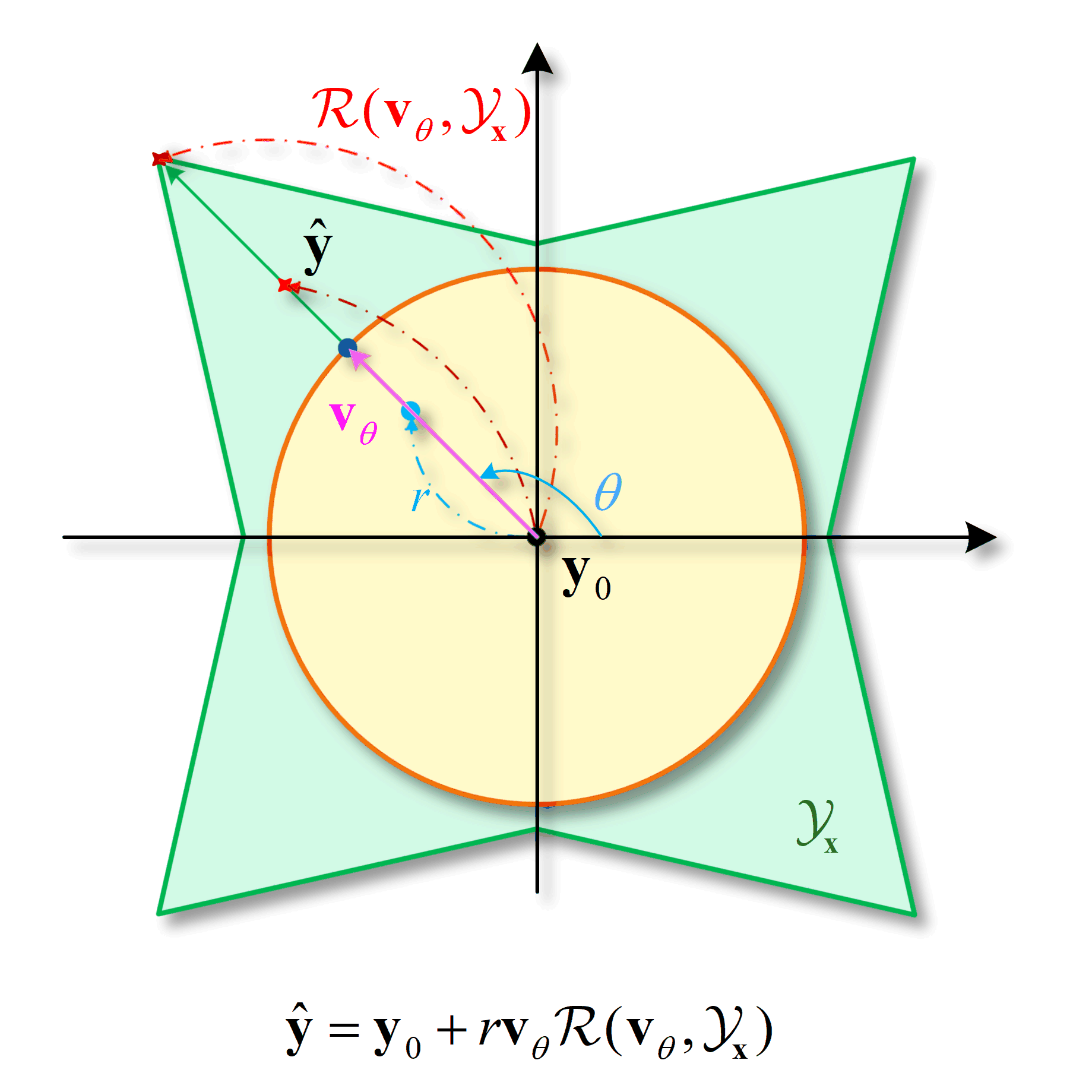}}
\vskip -0.2in
\caption{Illustration of the 2-D HoP principle: The larger green area represents the feasible region $\mathcal{Y}_\mathbf{x}$, while $\mathbf{y}_0\in\mathcal{Y}_\mathbf{x}$. NN in the HoP framework outputs the blue dot as an initial solution within unit circle, the yellow region, centered on $\mathbf{y}_0$ and constrained by a bounded activation function. The blue dot is then scaled along the direction specified by unit vector $\mathbf{v}_{\theta}$ to the red dot $\mathbf{\hat{y}}$. The scaling factor is defined as the ratio given in ${r\mathcal{R}(\mathbf{v}_{{\theta}},\mathcal{Y}_\mathbf{x})}$.}

\label{polar-mapping-2d}
\end{center}
\vspace{-3em}
\end{figure}
The one-dimensional case provides the essence of HoP. In the following higher-dimensional constraint sets, the extended homeomorphism $\mathcal{H}$ is introduced to meet the requirement. The primary reason for introducing homeomorphic mappings is to guarantee that every feasible variable can be bijectively mapped to the NN output. This one-to-one correspondence is critical for maintaining feasibility and ensuring that after homeomorphic mapping the NN outputs consistently satisfy the constraints.

\paragraph{Extension to the 2-D Case}
Building on the insight from the 1-D case, we extend the idea of bounded function based mappings to the 2-D case while preserving the properties of homeomorphic mappings. 

As shown in Fig.~\ref{polar-mapping-2d}, the feasible region $\mathcal{Y}_{\mathbf{x}}$ is the green region where $\mathbf{y}_0 \in \mathcal{Y}_{\mathbf{x}}$. We construct a unit circle, the yellow region, by polar coordinate system, with its origin centered at $\mathbf{y}_0$. Generally, $\mathbf{y}_0$ is obtained by solving a convex optimization problem over $\mathcal{Y}_{\mathbf{x}}$, where the objective is either a generic convex function or the Chebyshev center. Therefore, for $\forall \mathbf{\hat{y}} \in \mathcal{Y}_\mathbf{x}$, we have homeomorphic mapping $\mathcal{H}$ given as,
\begin{flalign}\label{2D_hop}
    \mathcal{H}\colon\quad&\mathbf{\hat{y}}=\mathbf{y}_0 + r\mathbf{v}_{{\theta}}\mathcal{R}(\mathbf{v}_{{\theta}},\mathcal{Y}_\mathbf{x})
\end{flalign}
where $r\in(0,1)$ is a scale, $\theta \in(0,2\pi)$ is the angle between x-axis and direction vector, $\mathbf{v}_{{\theta}} $ denotes the unit direction vector, and $\mathcal{R}(\mathbf{v}_{{\theta}},\mathcal{Y}_\mathbf{x})$ defines the distance from $\mathbf{y}_0$ to the boundary of $\mathcal{Y}_\mathbf{x}$ in the direction specified by $\theta$. The $\theta$, $r$, and $\mathbf{v}_{{\theta}} $ are defined as follows:
\begin{flalign}
    \begin{bmatrix}
    \theta\\
    r
    \end{bmatrix} =\begin{bmatrix}
    2\pi&0\\
    0&1
    \end{bmatrix}\begin{bmatrix}
   \mathcal{B}(z_\theta)\\
    \mathcal{B}(z_r)
    \end{bmatrix}, \quad
    \mathbf{v}_{{\theta}} = \begin{bmatrix}
    \cos{\theta}\\
    \sin{\theta}
    \end{bmatrix},
\end{flalign}
where $\mathbf{z} = [z_\theta, z_r]^T$ is raw NN output. It is worth noting that in optimization problems, redundant constraints, which do not affect the feasible region because they are implied by other constraints, can arise and significantly complicate the process of identifying boundary points. To address these challenges, it is critical to ensure that the homeomorphic mapping identifies the boundary points of the feasible region $\mathcal{Y}_\mathbf{x}$ accurately in the presence of redundant constraints. Our proposed method leverages the polar coordinate system to handle this issue effectively. When redundant constraints exist, $\mathcal{R}(\mathbf{v}_{{\theta}},\mathcal{Y}_\mathbf{x})$ finds boundary points by searching the closest intersection in the direction specified by $\theta$, as formalized in Proposition \ref{proposition_redundant}:
\begin{proposition} \label{proposition_redundant}
Let $ C_1, C_2, \dots, C_n $ be sets in the Euclidean space $\mathbb{R}^n$, and let their intersection $ C = \bigcap_{i=1}^N C_i $ be star-convex set. If $ \mathbf{y}_0 \in \operatorname{int}(C) $. For any ray originating from $ \mathbf{y}_0 $, the closest intersection point of the ray with $ C $ belongs to the set $C$ and lies on the boundary of $C$.
\end{proposition}
The proof of Proposition \ref{proposition_redundant} is given in Appendix \ref{proof_redundant}. As a result, the output $\mathbf{\hat{y}}$ computed using Eq. (\ref{2D_hop}) is guaranteed to be feasible for corresponding hard constraints $\mathcal{Y}_\mathbf{x}$. 

\paragraph{Extension to the Semi-Unbounded Case}

Since in the semi-unbounded problem, in which the feasible region extends indefinitely in some directions, making boundary determination challenging or ill-defined at infinity. $\mathcal{R}(\mathbf{v}_{{\theta}},\mathcal{Y}_\mathbf{x})$ given in 2-D scenario is impractical, as the intersection may not exist or could approach infinity. To address this issue, we introduce the following spherical homeomorphism mapping. The process begins with NN's raw outputs $\mathbf{z} = [\mathbf{z}_{{\theta}},{z}_{r}]^T, \mathbf{z}_{{\theta}}\in \mathbb{R}^{d},{z}_{r}\in \mathbb{R}$, where $d=2$ for current 2-D case while the following framework is also applicable to high-dimensional problem. To ensure the outputs are contained within a unit hypersphere, the transformations in Eqs. (\ref{recnection_theta}) and (\ref{recnection_r}) are applied to bound the raw outputs:
\begin{flalign}
    &\mathbf{v}_{{\theta}}  =
    \begin{cases}
    {\mathbf{{z}}_{{\theta}}}/{||\mathbf{{z}}_{{\theta}}||_2}, & \text{if } {z}_{r} \geq 0, \\
    {-\mathbf{{z}}_{{\theta}}}/{||\mathbf{{z}}_{{\theta}}||_2}, & \text{otherwise}.
    \end{cases}\label{recnection_theta}\\
    &\bar{z}_{r}  = \mathcal{B}(|{z}_{r}|),\label{recnection_r}
\end{flalign}
where $\mathbf{v}_{{\theta}}$ is direction vector, $\bar{z}_{r}$ is angle scale. Therefore, the (\ref{recnection_theta}) and (\ref{recnection_r}) can bound the output within unit hyper-sphere and avoid stagnation problem where we provide designing analyses in Section \ref{reconnection_theory}.

\begin{figure}[ht]
\begin{center}
\centerline{\includegraphics[width=0.8\columnwidth]{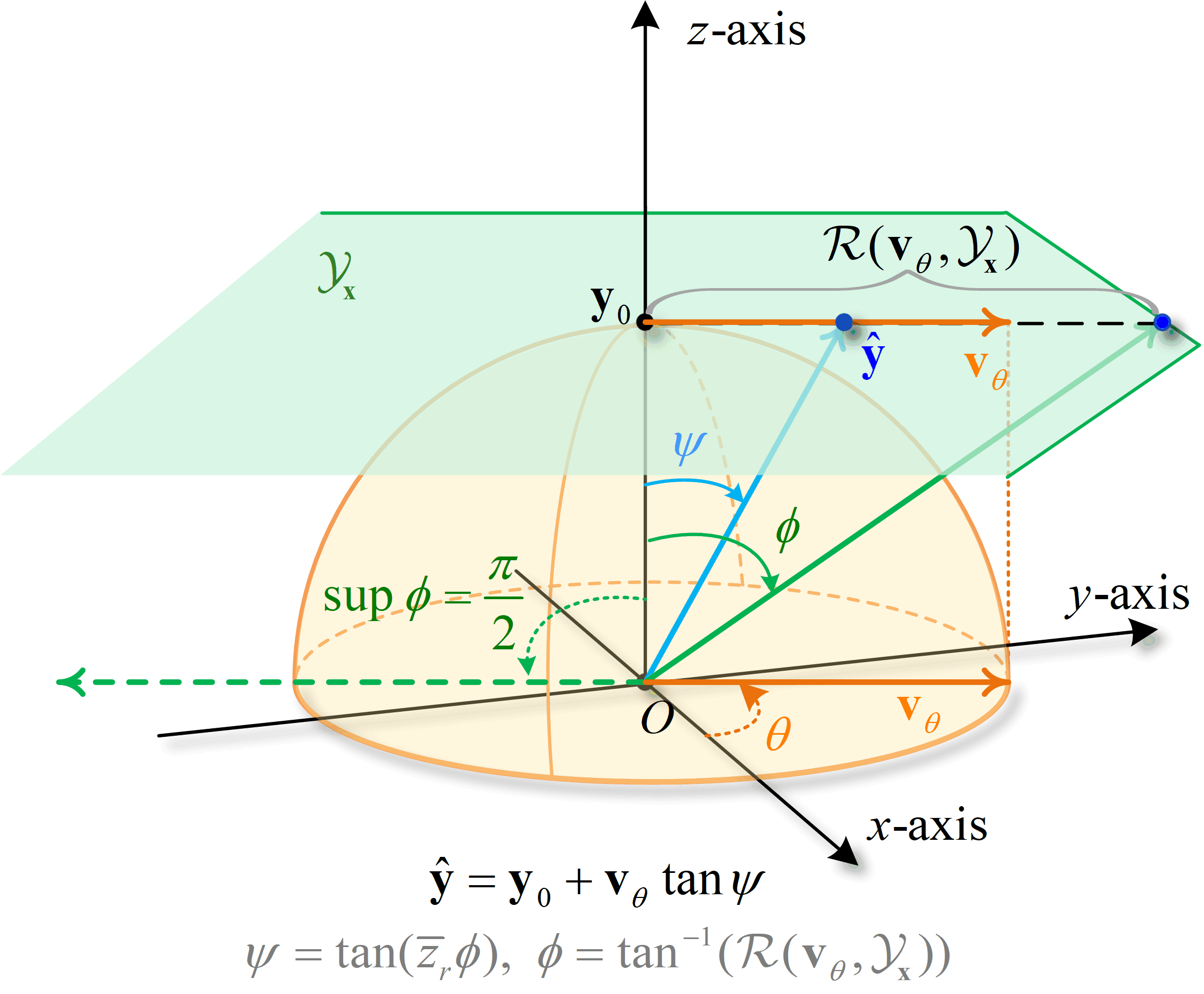}}
\vskip -0.1in
\caption{Sketch of the spherical coordinate transformation for semi-unbounded constraints. The 2-D plane (green plane) is elevated to a higher-dimensional system, where the distance $\mathcal{R} (\mathbf{v}_{\theta},\mathcal{Y}_{\mathbf{x}})$ in direction $\mathbf{v}_{\theta}$ from $\mathbf{y}_0$ to boundary, is mapped as the boundary angle $\phi$. Then NN's output ratio $\bar{z}_r$ and $\phi$ is transformed by Eq. (\ref{v_theta_transform}) to $\psi$, the inclination angle of blue ray, in horizontal direction $\mathbf{v}_{\theta}$. Finally, $\mathbf{\hat{y}}$ is the intersection of the blue ray and green plane. Furthermore, points at infinity in the green space correspond to the angle on equator where $\psi = \frac{\pi}{2}$.}
\label{polar-mapping-3d}
\end{center}
\vspace{-3em}
\end{figure}
Based on the polar sphere vector, $[\mathbf{v}_{\theta},\bar{z}_{r}]^T$, we set up polar coordinate system centered at $\mathbf{y}_0$. The 2-D homeomorphic mapping in Eq. (\ref{2D_hop}) can be extended to the HoP in semi-unbounded case and reformulated as follows:
\begin{flalign}\label{ND_hop}
    \mathcal{H}\colon\quad\mathbf{\hat{y}} = \mathbf{y}_0 +  \mathbf{v}_{\theta}\tan(\psi),
\end{flalign}
where  $ \psi$ is a angle defined as:
\begin{flalign} \label{v_theta_transform}
    &\psi = \bar{z}_r\phi,\quad\phi = \tan^{-1}(\mathcal{R}(\mathbf{v}_{\theta},\mathcal{Y}_\mathbf{x}))
\end{flalign}
To illustrate the essence of spherical mapping, we use 2-D optimization problem as an example. As shown in Fig.~\ref{polar-mapping-3d}, the original 2-D plane (green plane) is elevated into an additional $z$-axis dimension with unit distance from point O to $\mathbf{y}_0$. Here, $\mathbf{v}_{\theta}$ represents the unit direction vector of green plane. $\phi\in (0,\pi/2)$ represents the maximum angle between z-axis and the green ray extending to the dark green boundary in the direction defined by $\mathbf{v}_{\theta}$. The supremum of $\phi$, denoted as $\text{sup}\phi$, corresponds to the angle between z-axis and an asymptotic direction toward infinity. 


The blue ray defined by angle $\psi = \bar{z}_r\phi$, $\psi\in(0,\phi - \epsilon)$ with z-axis, is bounded by ratio $\bar{z}_r$, where $\epsilon$ is a small positive number to prevent divergence of the Jacobian determinant which is explained in Appendix~\ref{sec:phi_v_mapping}. Then, Eq. (\ref{ND_hop}) maps the angle $\psi$ to $\mathbf{\hat{y}}$ on the green plan which is the intersection of green plane and blue ray. Since $\phi$ is boundary angle while $\psi$ is bounded by $\phi$, the feasibility of intersection point $\mathbf{\hat{y}}$ is ensured.

By transforming the distance to angle, this approach can resolve the challenge of semi-unbounded regions. Building on the semi-unbounded design, HoP can be naturally extended to higher-dimensional and more general scenarios. The pseudocode for the HoP is presented in Algorithm \ref{alg:HoP_final_algorithm}.

\subsection{Resolving Stagnation in Polar Optimization via Reconnection}\label{reconnection_theory}

\begin{figure}[ht]
\begin{center}
\centerline{\includegraphics[width=0.9\columnwidth]{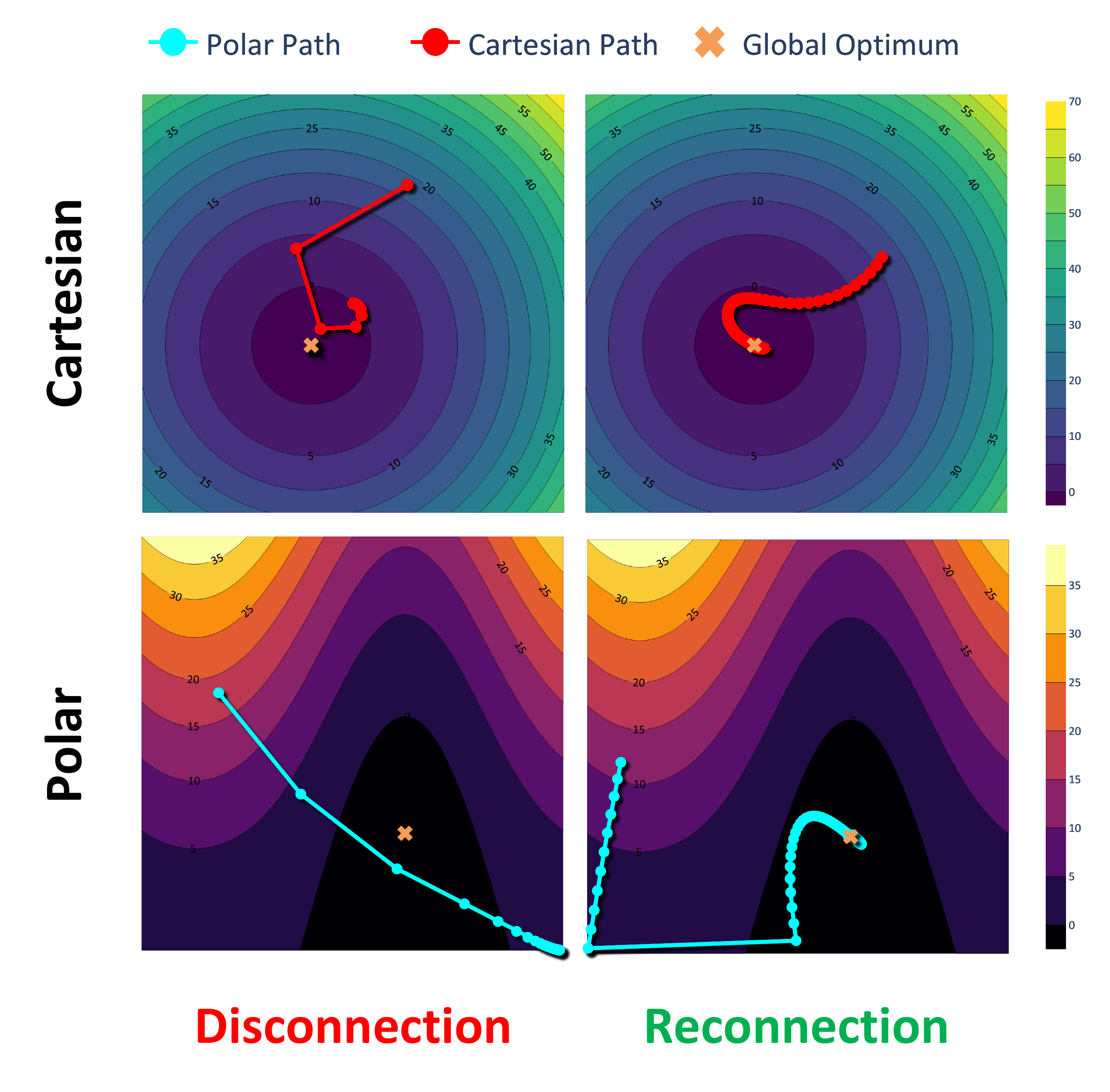}}
\vskip -0.2in
\caption{Comparison of optimization trajectories in Cartesian and polar coordinates. The left column demonstrates disconnection issues in polar coordinates, while the right column shows improved behavior with reconnection strategies.}
\label{polar-experiment}
\end{center}
\vskip -0.3in
\end{figure}

In this subsection, we analyze the reason for optimizer stagnation during training procedure of the polar optimization system, and provide two corresponding solutions, including the reconnection strategy given in Eqs. (\ref{recnection_theta}) and (\ref{recnection_r}). 
Firstly, since the use of polar coordinates for optimization brings unique challenges in radial and angular variables updating by simple bounded activation function. To facilitate understanding of the consequence of stagnation, we provide equivalent 2-D convergence sketch in Fig.~\ref{polar-experiment} to demonstrate the phenomenon. As illustrated in two left sub-figures of Fig. ~\ref{polar-experiment}, a critical issue is \textit{radial stagnation and angular freezing}, where the optimizer stagnates near $r = 0$ and unable to adjust the angular variable ${\theta}$. 

The most significant reason is the non-negativity constraint, $r\geq 0$. When $r$ reaches zero and hold the tendency to be negative, the radial updates are truncated. Meanwhile, the angular updates stagnate completely due to the vanishing gradient caused by $r=0$. Furthermore, this phenomenon is exacerbated when optimizer has large learning rate or momentum. Specifically, this stagnation is clearly observed in left sub-figure of Fig.~\ref{polar-experiment}. The detailed mathematical formulation of this phenomenon and theoretical analysis is provided in Appendix ~\ref{prop:radial_stagnation} and ~\ref{prop:LR_Momentum}.

To mitigate radial stagnation, we provide two alternative solution. The first one is dynamically scaled learning rate for radial updates where adjustable learning rate of $r$ avoid its updating truncation near $r = 0$. This strategy guarantees stability and prevents freezing, as rigorously analyzed in the Appendix ~\ref{prop:Dynamic_LR}.

Another more robust solution is geometric reconnection techniques, shown in bottom right sub-figure in Fig.~\ref{polar-experiment}. When $r < 0$, we expand the angular and radial domains with reconnecting the polar space by
\begin{flalign}
r \mapsto |r|, \quad \theta \mapsto \theta + \pi.
\end{flalign}
Moreover, to guarantee angular is continuous, we output $\theta\in \mathbb{R}$ which connects $\theta=0$ and $\theta=2\pi$ periodically. This reconnection design allows the optimizer to traverse across the regions that were truncated in previous polar space, enabling smoother transition and better exploration of the solution space. 


Therefore, the geometric reconnection strategies significantly enhance the stability and efficiency of optimization in polar coordinates. As shown in Fig.~\ref{polar-experiment} (bottom right), the proposed methods enable the optimizer to maintain smooth trajectories, effectively avoiding the disconnection issues observed in Fig.~\ref{polar-experiment} (bottom left). As a consequence, we apply similar reconnection design in Eq. (\ref{recnection_theta}) and Eq. (\ref{recnection_r}) to prevent truncation of $\mathbf{z}_{\theta}$ and $z_r$.

\begin{algorithm}[tb]
   \caption{HoP: Homeomorphic Polar Learning}
   \label{alg:HoP_final_algorithm}
\begin{algorithmic}
   \STATE Prepare parameters $\mathcal{X}_{\text{train}}$ and $\mathcal{X}_{\text{test}}$, initialize NN, polar center $\mathbf{y}_0$ for every instance in $\mathcal{X}_{\text{train}}$ and $\mathcal{X}_{\text{test}}$.
   \STATE TRAIN:
   \FOR{every epoch}
   \FOR{every batch data $\mathbf{x}$ in $\mathcal{X}_{\text{train}}$}
   \STATE Update $[\mathbf{{z}}_{\theta},{z}_r]^T = \text{NN}(\mathbf{x})$;
   \STATE Update $\mathbf{v}_{\theta}$ by Eq. (\ref{recnection_theta}), ${\bar{z}_r}$ by Eq. (\ref{recnection_r});
   \STATE Output estimated variables $\mathbf{\hat{y}}$ by Eq. (\ref{ND_hop});
   \STATE Update loss function $\mathcal{L}(\mathbf{\hat{y}})$;
   \STATE Backward propagation to update NN parameters.
   \ENDFOR
   \ENDFOR
   \STATE TEST:
   \FOR{every batch data $\mathbf{x}$ in $\mathcal{X}_{\text{test}}$}
   \STATE Update $[\mathbf{{z}}_{\theta},{z}_r]^T  = \text{NN}(\mathbf{x})$;
   \STATE Update $\mathbf{v}_{\theta}$ by Eq. (\ref{recnection_theta}), ${\bar{z}_r}$ by Eq. (\ref{recnection_r});
   \STATE Output estimated variables $\mathbf{\hat{y}}$ by Eq. (\ref{ND_hop});
   \STATE Compute gap between $\mathcal{L}(\mathbf{\hat{y}})$ and $\mathcal{L}(\mathbf{{y}}^*)$;
   \ENDFOR
\end{algorithmic}
\end{algorithm}
\vskip -0.1in
\section{Experiments}
In following experiments we evaluate HoP by comparing with other methods as baselines in three aspects, including optimality, feasibility and computation efficiency. Comparative experiments with traditional optimizers, DC3 \cite{donti2021dc3} and other NN-based L2O approaches are conducted to validate HoP's effectiveness. All NNs follow a uniform architecture: a 3-layer multilayer perceptron (MLP) with ReLU activation functions. Since other NN-based methods such as DC3 typically rely on penalty functions to handle constraints, where the penalty is defined as: 
\begin{flalign}
    \mathcal{P}(\mathbf{\hat{y}}, \mathcal{Y}_\mathbf{x}) = \mathbb{I}(\mathbf{\hat{y}}\notin \mathcal{Y}_\mathbf{x})\cdot\text{dist}(\mathbf{\hat{y}}, \mathcal{Y}_\mathbf{x})
\end{flalign}
where $\mathbb{I}(\hat{\mathbf{y}}\notin\mathcal{Y}_\mathbf{x})$ indicates constraint violations, and $\text{dist}(\mathbf{\hat{y}}, \mathcal{Y}_\mathbf{x})$ quantifies the distance to the constraint set. 

\begin{itemize}
    \item Optimizer: For synthetic benchmarks (Experiments \ref{exp:Synthetic}), the Sequential Least Squares Programming (SLSQP) is used as the solver. For the quality-of-service-aware multi-input-single-output communication system weighted sum rate (QoS-MISO WSR) problem in Experiment \ref{miso_prob}, Splitting Conic Solver (SCS) \cite{diamond2016cvxpy} and fractional programming (FP) \cite{shen2018fractional} is applied for alternative optimization as the solver baseline.
    
    \item NN-SSL (Self-Supervised Learning): The loss function consists individually of the objective function without any penalty term, defined as $\mathcal{L}(\mathbf{\hat{y}}) = f_{\mathbf{x}}(\mathbf{\hat{y}})$.
    
    \item NN-SL: The loss minimizes the mean squared error (MSE) between predictions $\hat{\mathbf{y}}$ and the target labels $\mathbf{y}^*$, expressed as $\mathcal{L}(\mathbf{\hat{y}}) = (\mathbf{\hat{y}} - \mathbf{y}^*)^2$.
    
    \item NN-SL-SC (Soft-Constraint + Supervised Learning): The loss function integrates the MSE and a soft constraint penalty, expressed as $\mathcal{L}(\mathbf{\hat{y}}) = (\mathbf{\hat{y}} - \mathbf{y}^*)^2 + \lambda \mathcal{P}(\mathbf{\hat{y}}), \lambda\geq 0$.
    
    \item NN-SSL-SC (Soft-Constraint + Self-Supervised Learning): The loss incorporates both the objective function and a soft constraint penalty: $\mathcal{L}(\mathbf{\hat{y}}) = f_{\mathbf{x}}(\mathbf{\hat{y}}) + \lambda \mathcal{P}(\mathbf{\hat{y}})$.
    
    \item DC3: This baseline follows \cite{donti2021dc3}, using the soft-constraint loss of NN-SSL-SC with additional gradient post-corrections to improve feasibility.
\end{itemize}
To ensure a fair and comprehensive evaluation, we utilize the same metrics proposed in DC3. Specifically, the metrics in Tables \ref{Polygon-table} -- \ref{miso-table} are defined as follows: (1) Obj. Value represents the objective function value, $f_{\mathbf{x}}(\mathbf{\hat{y}})$, achieved by the solver. (2) Max. Cons and Mean. Cons denote the $\text{max}{\mathcal{P}(\mathbf{\hat{y}})}$ and $\text{mean}({\mathcal{P}(\mathbf{\hat{y}})})$, respectively. (3) Vio. Rate is violation rate which measures the percentage of predicted infeasible solutions. (4) Time reflects the computational efficiency of each method. Moreover, the blue numbers in following tables denote the results from HoP, while the red numbers indicate the worst results or violation for corresponding metrics.
Further experimental setup details are provided in Appendix \ref{exp_setting_details}.

\subsection{Synthetic Benchmarks} \label{exp:Synthetic}

\begin{table*}[ht]
\caption{Experimental results on an 8-sided polygon constraint with varying constraints, using the sinusoidal QP as the objective function. The results demonstrate the performance of different methods in terms of objective value, constraint satisfaction, violation rate, and computation time. }
\label{Polygon-table}
\vskip 0.1in
\begin{center}
\begin{small}
\begin{sc}
\begin{tabular}{lrrrrr}
\toprule
Method & Obj. Value $\downarrow$ & Max. Cons $\downarrow$ & Mean. Cons $\downarrow$ & Vio. Rate $\downarrow$ & Time / \SI{}{\milli\second} $\downarrow$\\
\midrule
Optimizer& {-29.7252} & 0.0000 & 0.0000 & 0.00\% & {\color{red}0.07104} \\ 
HoP & {\color{blue}-29.7170}& {\color{blue}0.0000}& {\color{blue}0.0000} &{\color{blue}0.00\%} & {\color{blue}0.00444}\\ 
NN-SSL   & -29.7481 & {\color{red} 0.0334} & {\color{red}0.0029} &{\color{red}18.10\%}&0.00004 \\ 
NN-SL   & -29.7247 & {\color{red}0.0153} & {\color{red}0.0001} & {\color{red}7.47\%}&0.00004 \\ 
NN-SSL-SC   & {\color{red}-29.5601} & {\color{red}0.0108} & {\color{red}0.0001} & {\color{red}2.94\%}&0.00004\\ 
NN-SL-SC   & -29.7203 & {\color{red}0.0080} & {\color{red}0.0001} & {\color{red}0.35\%}&0.00004 \\ 
DC3    & -29.6893 & 0.0000 & 0.0000 & 0.00\% & 0.01491 \\ 
\bottomrule
\end{tabular}
\end{sc}
\end{small}
\end{center}
\vskip -0.2in
\end{table*}

To evaluate Hop, we consider three different synthetic problems: the polygon-constrained problem and the $\ell_p$-norm problem and the high-dimensional semi-unbounded problem. The polygon-constrained problem serves as the primary validation for convergence, the $\ell_p$-norm problem evaluates the capability in addressing star-convex scenarios, and the high-dimensional semi-unbounded problem tests the method’s scalability and effectiveness in handling complex constraints in higher dimensions.
\paragraph{(a) Polygon-Constrained Problem}
As the first benchmark experiment, we choose sinusoidal quadratic programming (QP) as the non-convex objective function with linear constraints. The problem is formulated as follows:
\begin{flalign}\label{QP_sin_Poly}
    \min_{\mathbf{y}} \quad\frac{1}{2}\mathbf{y}^T\mathbf{Q}\mathbf{y} + \mathbf{p}^T\sin(\beta\mathbf{y}),
    \quad \text{s.t.}\quad \mathbf{A}\mathbf{y} \preceq \mathbf{b},
\end{flalign}
where matrix $\mathbf{Q}$ is a positive semi-definite matrix, vector $\mathbf{p}$ is a parameter vector, and scalar $\beta$ controls the frequency of the sinusoidal terms, which introduces non-convexity into the objective function. The constraint are defined by matrix $\mathbf{A}$ and vector $\mathbf{b}$. In this problem, the parameter $\mathbf{x}$ is $\mathbf{b}$, consistent with the setup in \cite{donti2021dc3}. 


Table \ref{Polygon-table} presents the results for Problem (\ref{QP_sin_Poly}) under an 8-sided polygon constraint. HoP achieves perfect constraint satisfaction, 0\% violation rate, while significantly outperforming NN-based methods, which exhibit violation rates up to 18.10\%. In terms of computational efficiency, HoP is over 15× faster than the optimizer. Moreover, objective value, -29.7170, given by HoP closely approaches the optimizer's -29.7252 and surpasses DC3's -29.6893. 

These results demonstrate HoP’s ability to enforce hard constraints rigorously, achieve nearer optimal solutions, and operate with superior computational efficiency in simple 2-D optimization problem which validate its effectiveness.

\begin{table*}[ht]
\caption{Experimental results on the $\ell_{0.5}$-norm constraint with $b=1$, using the QP as the objective function. The results demonstrate the performance of different methods in terms of objective value, constraint satisfaction, violation rate, and computation time.}
\label{lp-table}
\vskip 0.1in
\begin{center}
\begin{small}
\begin{sc}
\begin{tabular}{lrrrrr}
\toprule
Method & Obj. Value $\downarrow$ & Max. Cons $\downarrow$ & Mean. Cons $\downarrow$ & Vio. Rate $\downarrow$ & Time / \SI{}{\milli\second} $\downarrow$\\
\midrule
Optimizer & {-0.4824} & 0.0000 & 0.0000 & 0.00\% & {\color{red}1.49371} \\ 
HoP & {\color{blue}-0.3886} & {\color{blue}0.0000}& {\color{blue}0.0000} &{\color{blue}0.00\%} &{\color{blue}0.00946}\\
NN-SSL   & -1.2227 & {\color{red}67.6524} & {\color{red}17.5511} & {\color{red}100.00\%} &0.00886\\
NN-SL   & -0.5285 & {\color{red}10.2304} & {\color{red}2.1575} & {\color{red}99.96\%} &0.00886\\ 
NN-SSL-SC   & -0.0679 & {\color{red}0.4356} & {\color{red}0.0007} & {\color{red}0.75\%} &0.00886\\ 
NN-SL-SC   & {\color{red}-0.0132} & {\color{red}0.3018} & {\color{red}0.0002} & {\color{red}0.02\%} &0.00886\\ 
DC3    & -0.0730 & 0.0000 & 0.0000 & 0.00\% & 0.02295       \\ 
\bottomrule
\end{tabular}
\end{sc}
\end{small}
\end{center}
\vskip -0.2in
\end{table*}
\begin{table*}[t]
\caption{Results on high-dimension problem for $20$ variables with $20$ linear constraints. We use $14,000$ instances for training while $6,000$ samples for testing. The constraints are fixed in this problem while the objective function is different in each instance.}
\label{Table:high_dim}
\vskip -0.15in
\begin{center}
\begin{small}
\begin{sc}
\begin{tabular}{lrrrrr}
\toprule
Method & Obj. Value $\downarrow$ & Max. Cons $\downarrow$ & Mean. Cons $\downarrow$ & Vio. Rate $\downarrow$ & Time / \SI{}{\milli\second} $\downarrow$ \\
\midrule
Optimizer& {-7.6901}& 0.0000& 0.0000 & 0.00\%&{\color{red}4.87273}\\ 
Hop &{\color{blue}{-4.7683}}& {\color{blue}0.0000}& {\color{blue}0.0000} & {\color{blue}0.00\%} &{\color{blue}0.09262}\\ 
NN-SSL   & 0.2772&{\color{red}0.2079}& {\color{red}0.0115} & {\color{red}100.00\%} &0.04763\\ 
NN-SL   &29.5046& 0.0000& 0.0000 & 0.00\% &0.04742\\ 
NN-SSL-SC   & 17.3966&{\color{red}0.1758}& {\color{red}0.0001} & {\color{red}22.18\%} &0.05047\\ 
NN-SL-SC   & {\color{red}43.9163}& 0.0000& 0.0000&0.00\%&0.04812\\ 
DC3    & 8.9205& 0.0000&   0.0000&0.00\%  &  0.08768\\ 
\bottomrule
\end{tabular}
\end{sc}
\end{small}
\end{center}
\vskip -0.2in
\end{table*}
\paragraph{(b) $\ell_p$-norm Problem}
The second problem, an $\ell_p$-norm problem, features a QP objective function and star-convex constraints:
\begin{flalign}\label{eq:lp-norm}
\min_{\mathbf{y}}\quad \frac{1}{2} \mathbf{y}^T \mathbf{Q} \mathbf{y} + \mathbf{p}^T \mathbf{y},\quad\text{s.t.}\quad ||\mathbf{y}||_{\ell_{p}}^p \leq {b}
\end{flalign}
where vector $\mathbf{p}$ serves as the input variable $\mathbf{x}$ defined in Eq. (\ref{eq:lp-norm}) to the problem.  Unlike the linear and convex constraints in Section \ref{exp:Synthetic}(a), this problem introduces non-linear, non-convex star-convex constraints, which present a more challenging optimization scenario. 

As shown in Table \ref{lp-table}, leveraging on topological equivalence given by HoP's homeomorphic mapping, HoP achieves perfect constraint satisfaction (0\% violation rate) with an objective value of -0.3886, significantly outperforming DC3's -0.0730. In contrast, DC3's gradient post-corrections often introduce bias by diverging from the objective-minimizing direction, leading to suboptimal solutions. Especially in the $\ell_p$ norm optimization, its complex constraint boundary brings more suboptimal corrections to DC3's gradient post-corrections, which highly rely on the gradient of the constraints.

Computationally, HoP is over 150$\times$ faster than the traditional optimizer and matches the speed of NN-based methods. Unlike other basic NN-based approaches which exhibit violative results and inferior objective values, HoP strictly enforces constraints while delivering superior solution quality. These results highlight HoP's effectiveness in handling star-convex constraints.

\vspace{-0.5em}
\paragraph{(c) High-Dimensional Semi-Unbounded Problem}

\begin{table*}[ht]
\caption{Results on QoS-MISO WSR problem for $U=3$ and $M=4$. The training set has $2,800$ instances while test set has $1,200$ instances. Objective function is MISO WSR, where constraints are QoS and power limitation.}
\label{miso-table}
\vskip -0.15in
\begin{center}
\begin{small}
\begin{sc}
\begin{tabular}{lrrrrr}
\toprule
Method & Obj. Value $\uparrow$ & Max. Cons $\downarrow$ & Mean. Cons $\downarrow$ & Vio. Rate $\downarrow$ & Time / \SI{}{\milli\second} $\downarrow$ \\
\midrule
Optimizer& ${1.3791}$& 0.0000&0.0000&0.00\% &{\color{red}11.31842}\\ 
HoP & {{\color{blue}1.1393}}& {\color{blue}0.0000}& {\color{blue}0.0000}& {\color{blue}0.00\%}&{\color{blue}0.99895}\\ 
NN-SSL   & 20.6766& {\color{red}1738.1122}&{\color{red}395.4313} &{\color{red}100.00\%}& 0.01723\\ 
NN-SL   & 3.2477& {\color{red}0.1636}& {\color{red}0.0559}& {\color{red}100.00\%} &0.01984\\ 
NN-SSL-SC   & 0.4150& {\color{red}0.2886}& {\color{red}0.0029}& {\color{red}18.25\%}&0.01834\\ 
NN-SL-SC   & {\color{red}0.3244}& {\color{red}0.7794}& {\color{red}0.0035} &{\color{red}23.00\%}&0.01864\\ 
DC3    & 0.3381& 0.0000&  0.0000&0.00\% &{5.53113}     \\ 
\bottomrule
\end{tabular}
\end{sc}
\end{small}
\end{center}
\vskip -0.1in
\end{table*}



This experiment demonstrates the ability of Hop to handle high-dimensional and semi-unbounded problem. As same as the previous experiments, the problem follows the sinusoidal QP formulation introduced earlier in Eq. (\ref{QP_sin_Poly}). Here, $\mathbf{p}$ represents the input parameter $\mathbf{x}$ for the NN model. As shown in Table \ref{Table:high_dim},  HoP consistently satisfies all constraints without any violations, confirming its robustness and feasibility under these challenging conditions. Moreover, HoP achieves the nearly optimal objective value -4.7683 among all methods, approaching to the traditional optimizer -7.6901 while maintaining strict constraint satisfaction. Otherwise, HoP still hold the advantages on computation complexity which is 52$\times$ faster than optimizer. 

Notably, NN-SL and NN-SL-SC also achieve 0 violation rates in this high-dimensional semi-unbounded experiment. This is likely due to the expanded feasible region in such scenarios, where constraints are less tight, and the optimal solution lies away from the boundary. However, it does not imply that SL-based methods can theoretically guarantee hard constraint satisfaction under all conditions. 

\begin{figure}[ht]
\begin{center}
\centerline{\includegraphics[width=0.7\columnwidth]{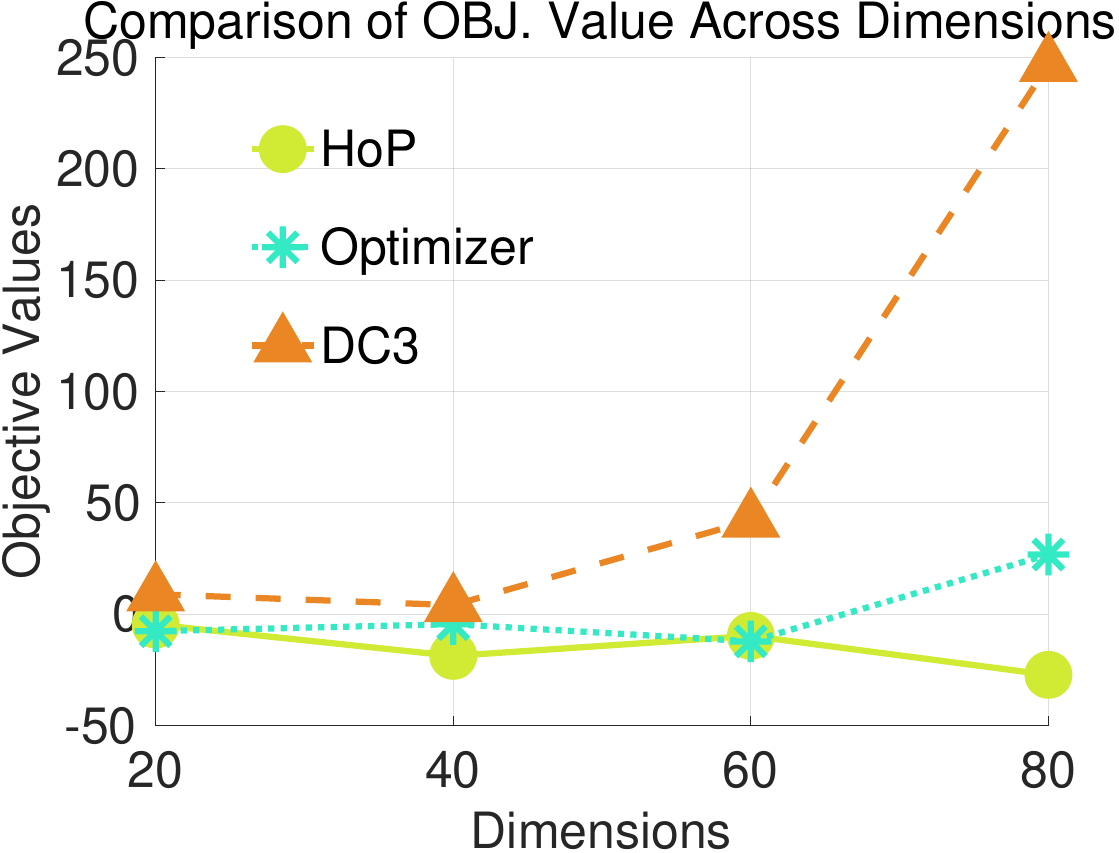}}
\caption{Comparison of objective values across methods under different dimensional settings.}
\label{dim_comp}
\end{center}
\vspace{-3em}
\end{figure}

To further evaluate the performance of different methods under higher-dimensional settings, we conduct analyses across 20, 40, 60, and 80 dimensions. As shown in Fig. \ref{dim_comp}, HoP consistently performs near the optimizer in terms of objective value and exhibits a trend better than the optimizer's performance as dimensionality increases. In contrast, the DC3 method struggles to surpass the optimizer's objective value and shows a clear decline in performance as the dimensionality rises. These results highlight the robustness and scalability of HoP, even in high-dimensional scenarios.

\subsection{Application -- QoS-MISO WSR Problem}\label{miso_prob}
In this experiment, we implement HoP to address the QoS-MISO WSR problem. The QoS-MISO WSR problem is a well-known NP-hard problem with non-linear constraints in communication engineering \cite{tang2023energy,niu2021qos}. In most studies on the QoS-MISO WSR problem, researchers commonly employ alternative optimization methods to obtain solutions, which often require significant computational resources. The problem is formulated as follows: 
\begin{subequations}\label{WSR_pro}
\vskip -0.3in
\begin{flalign}
    \max_{\mathbf{w}_k} &\sum_{k=1}^U \alpha_k \log_{2}(1+\text{SINR}_k)\tag{\ref{WSR_pro}}\\
    \text{s.t.}\
    &\log_{2}(1+\text{SINR}_k)\geq \delta_k\label{QoS_cons}\\
    &\sum_{k=1}^U \text{Tr}(\mathbf{w}_k\mathbf{w}_k^H)+\text{P}_\text{c} \leq \text{P}_\text{max}\label{power}
\end{flalign}
\end{subequations}
where $\mathbf{w}_k \in \mathbb{C}^{M}$ is the beamformer, $(.)^H$ represents Hermitian transpose, $\mathbf{h}_k\in \mathbb{C}^{M}$ denotes channel state information, and $\sigma_k^2 \in \mathbb{R}$ is channel noise power.  $\alpha_k\in \mathbb{R}$ denotes priority weight for $U$ users, $\delta_k$ represents the QoS requirements. $\text{P}_\text{max}$ and $\text{P}_\text{c}$ define the system maximum power and circuit power consumption, respectively. The signal-interference-noise-ratio $\text{SINR}_k$ is defined as 
\begin{flalign}
    &\text{SINR}_k = \frac{\mathbf{w}_k^H\mathbf{h}_k\mathbf{h}_k^H\mathbf{w}_k}{\sum_{j\neq k}^U \mathbf{w}_j^H\mathbf{h}_k\mathbf{h}_k^H\mathbf{w}_j+\sigma_k^2}
\end{flalign}
Given that wireless resource allocation requires real-time optimization strategies to effectively manage wireless resources with diverse channel state information, we select $\mathbf{h}_k$ as the input of NN. This selection is crucial as $\mathbf{h}_k$ significantly influences both the constraints and the objective function in Eq. (\ref{WSR_pro}). For a comprehensive understanding of how we apply HoP to the QoS-MISO WSR problem, please refer to Appendix \ref{QOSMISO_polar}.

According to Table \ref{miso-table}, both HoP and DC3 achieve perfect constraint satisfaction, with $0\%$  violation rate. However, the mechanisms of penalty augmentation and post-correction in DC3 create a multi-objective optimization dilemma, fundamentally impeding it to simultaneously guarantee solution feasibility and pursue optimality. This inherent conflict between feasibility and objective minimization manifests as a optimality gap degradation in both of MISO and high-dimensional scenarios, whereas HoP eliminates such competing objectives through homeomorphic mapping. NN-based methods are consistent with previous experiments, struggling with constraint satisfaction. NN-SSL and NN-SL have extremely high Max, Cons and a Mean. Cons, with $100\%$ violation rate, showing that it fails to guarantee feasibility. NN-SSL-SC and NN-SL-SC reduce constraint violations slightly, with violation rate of $18.25\%$ and $23.00\%$, respectively. But they still fall far short of the perfect satisfaction achieved by HoP and DC3. For computation efficiency, HoP is slightly slower than other NN solvers, while it still gains a huge speedup -- 10 times faster than traditional algorithm and 5 times faster than DC3.

The results clearly showcase the superiority of the HoP in balancing objective optimization and constraint satisfaction. While the Optimizer (SCS+FP) achieves the best objective value, HoP obtains a very comparative  result, which outperforms other NN solvers. Moreover, HoP maintains perfect constraint satisfaction, matching the performance of DC3 in violation control. Thus, HoP is a highly effective and reliable method for solving the problem in engineering such as the QoS-MISO WSR problem.

\section{Conclusion}
In this work, we propose HoP, a novel L2O framework for solving hard-constrained optimization problems. The proposed architecture integrates NN predictions with a homeomorphic mapping, which transforms NN's outputs from spherical polar space to Cartesian coordinates, ensuring solution feasibility without extra penalties or post-correction.  Through extensive experiments encompassing both synthetic benchmark tasks and real-world applications, we demonstrate that HoP consistently outperforms existing L2O solvers, achieving superior optimality while maintaining zero constraint violation rates.

\bibliography{lib}
\bibliographystyle{icml2025}

\newpage
\appendix
\onecolumn

\section{Proof of Proposition \ref{proposition_redundant}}\label{proof_redundant}


\begin{proof} 
Since $ C $ is star-convex with respect to $ \mathbf{y}_0 $, for any point $ \mathbf{y} \in C $, the line segment connecting $ \mathbf{y}_0 $ and $ \mathbf{y} $ is entirely contained in $ C $. That is, for all $ t \in [0, 1] $, we have:
\begin{flalign}
   (1-t)\mathbf{y}_0 + t\mathbf{y} \in C.
\end{flalign}
Then we consider a ray originating from $ \mathbf{y}_0 $ in the direction of a unit vector $ \mathbf{v}_{\theta} \in \mathbb{R}^n $. The ray can be parameterized as:
\begin{flalign}
   R = \{ \mathbf{y}_0 + t \mathbf{v} : t \geq 0 \},
\end{flalign}
where $ t \geq 0 $ is the parameter along the ray. Since $ \mathbf{y}_0 \in \operatorname{int}(C) $, the ray $ R $ starts inside $ C $. By the Definition \ref{star_convex}, the ray $ R $ must intersect $ C $, and the portion of the ray close to $ \mathbf{y}_0 $ is entirely contained in $ C $. Let $ t^* $ be the supremum of the set of parameters $ t \geq 0 $ such that $ \mathbf{y}_0 + t v \in \operatorname{int}(C) $ which is given as:
\begin{flalign}
   t^* = \sup \{ t \geq 0 : \mathbf{y}_0 + t v \in \operatorname{int}(C) \}.
\end{flalign}
   Since $ \operatorname{int}(C) $ is open and $ \mathbf{y}_0 \in \operatorname{int}(C) $, this set is non-empty and $ t^* $ exists. Define:
\begin{flalign}
   \mathbf{y}_1 = \mathbf{y}_0 + t^* v.
\end{flalign}
By construction, the point $ y_1 $ satisfies the following: (1) For any $ t < t^* $, $ \mathbf{y}_0 + t \mathbf{v} \in \operatorname{int}(C) $; (2) For any $ t > t^* $, $ \mathbf{y}_0 + t \mathbf{v} \notin C $. Therefore, $ \mathbf{y}_1 $ lies on the boundary of $ C $, i.e., $ \mathbf{y}_1 \in \partial C $. Furthermore, since the ray is continuous and $ C $ is closed (as the finite intersection of closed sets), $ \mathbf{y}_1 \in C $.
For any ray originating from $ \mathbf{y}_0 $, the closest intersection point $ \mathbf{y}_1 $ belongs to $ C $ and lies on its boundary $ \partial C $.

\end{proof}

\section{Jacobian Analysis and Measure Distortion}
\label{sec:phi_v_mapping}

\noindent
This section presents a detailed analysis of the transformation
\begin{flalign}
\label{eq:phi_v_mapping}
\mathbf{\hat{y}}(\psi, \mathbf{v}_{\theta}) \;=\; \mathbf{y}_0 + \mathbf{v}_{\theta}\tan(\psi),
\end{flalign}
where $\psi \in [0, \tfrac{\pi}{2})$ is a scalar parameter, and $\mathbf{v}_{\theta} \in \mathbb{R}^d$ is a unit vector (i.e., $\|\mathbf{v}_{\theta}\|=1$). In this formulation, the direction $\mathbf{v}_{\theta}$ is treated as a free variable on the unit sphere $S^{d-1}$, while $\psi$ governs the radial displacement via the function $\tan(\psi)$. The mapping in Eq. \eqref{eq:phi_v_mapping} corresponds to the higher-dimensional homeomorphic mapping introduced in Eq.~\eqref{ND_hop}. It provides a diffeomorphic embedding of the parameter space $\mathcal{P}$ into $\mathbb{R}^d$, where $\mathcal{P}$ is defined as:
\begin{flalign}
\mathcal{P} \;=\; \bigl\{(\psi, \mathbf{v}_{\theta}) \mid \psi \in [0, \tfrac{\pi}{2}),\; \mathbf{v}_{\theta}\in S^{d-1}\bigr\}, \notag
\end{flalign}

The analysis below derives the Jacobian determinant of this transformation and discusses the measure distortion that arises as $\psi$ approaches $\tfrac{\pi}{2}$, where $\tan(\psi)$ diverges. These results provide insight into the geometric and numerical properties of the homeomorphic mapping when applied to optimization tasks in semi-unbounded domains.

\subsection{Jacobian Determinant}
To characterize local volume distortion, we compute the Jacobian determinant \cite{spivak2018calculus}. The mapping 
$\mathbf{\hat{y}}(\psi, \mathbf{v}_{\theta})$ is defined in Eq.~\eqref{eq:phi_v_mapping}. The total derivative 
$\mathrm{D}\mathbf{\hat{y}}(\psi,\mathbf{v}_{\theta})$ is a $d \times d$ matrix whose columns represent the partial derivatives of $\mathbf{\hat{y}}$ with respect to $\psi$ and to the $(d-1)$ degrees of freedom on the unit sphere $S^{d-1}$. Specifically:

\paragraph{(a) Derivative w.r.t.\ $\psi$.} For fixed $\mathbf{v}_{\theta}$,
\begin{flalign}
  \frac{\partial}{\partial \psi}\,\bigl(\tan(\psi)\,\mathbf{v}_{\theta}\bigr)
  \;=\;
  \sec^2(\psi)\,\mathbf{v}_{\theta}.
\end{flalign}
\paragraph{(b) Derivatives w.r.t.\ the sphere parameters $\mathbf{v}_{\theta}$.} 

\begin{lemma}[Tangent Space]
\label{lemma_tangent_space}
Let $\mathbf{v}_{\theta} \in S^{d-1}$ be a unit vector on the sphere in $\mathbb{R}^d$, satisfying $\|\mathbf{v}_{\theta}\|=1$. Then any infinitesimal variation $\mathrm{d}\mathbf{v}_{\theta}$ must lie in the tangent space $T_{\mathbf{v}_{\theta}}S^{d-1}$, given by:
$$
  \mathbf{v}_{\theta} \cdot \frac{\mathrm{d}\mathbf{v}_{\theta}}{\mathrm{d}\theta} = 0
  \quad\Longrightarrow\quad
  \mathrm{d}\mathbf{v}_{\theta} \in T_{\mathbf{v}_{\theta}}S^{d-1}.
$$
The tangent space $T_{\mathbf{v}_{\theta}}S^{d-1}$ is a $(d-1)$-dimensional subspace of $\mathbb{R}^d$.
\end{lemma}

\begin{proof}
The constraint $\|\mathbf{v}_{\theta}\| = 1$ implies $\mathbf{v}^T_{\theta} \cdot \mathbf{v}_{\theta} = 1$. Differentiating this equation with respect to $\theta$ yields:
$$
\frac{\mathrm{d}}{\mathrm{d}\theta} \bigl(\mathbf{v}^T_{\theta} \cdot \mathbf{v}_{\theta}\bigr) = 2\,\mathbf{v}_{\theta} \cdot \frac{\mathrm{d}\mathbf{v}_{\theta}}{\mathrm{d}\theta} = 0.
$$
Thus, any allowed variation $\frac{\mathrm{d}\mathbf{v}_{\theta}}{\mathrm{d}\theta}$ is orthogonal to $\mathbf{v}_{\theta}$, and therefore lies in the tangent space $T_{\mathbf{v}_{\theta}}S^{d-1}$, which is the subspace of $\mathbb{R}^d$ orthogonal to $\mathbf{v}_{\theta}$. 

Since $S^{d-1}$ is a $(d-1)$-dimensional manifold embedded in $\mathbb{R}^d$, its tangent space $T_{\mathbf{v}_{\theta}}S^{d-1}$ is also $(d-1)$-dimensional.
\end{proof}

\begin{lemma}[Linear Independence of $\{\mathbf{v}_{\theta}, \mathbf{w}_1, \dots, \mathbf{w}_{d-1}\}$]
\label{lemma_linear_independence}
Let $\mathbf{v}_{\theta} \in S^{d-1}$ be a unit vector, and let $\{\mathbf{w}_1, \dots, \mathbf{w}_{d-1}\}$ be an orthonormal basis of $T_{\mathbf{v}_{\theta}}S^{d-1}$, satisfying:
$$
 \langle \mathbf{v}_{\theta}, \mathbf{w}_i \rangle= 0, \quad \langle \mathbf{w}_i, \mathbf{w}_j \rangle=
  \begin{cases}
    1, & \text{if } i = j, \\
    0, & \text{if } i \neq j,
  \end{cases}
  \quad
  i,j = 1,\dots,d-1.
$$
Then the set $\{\mathbf{v}_{\theta}, \mathbf{w}_1, \dots, \mathbf{w}_{d-1}\}$ is linearly independent in $\mathbb{R}^d$. Specifically, any linear combination
$$
  a\,\mathbf{v}_{\theta} + \sum_{i=1}^{d-1} b_i\,\mathbf{w}_i = \mathbf{0}
$$
implies $a = 0$ and $b_i = 0$ for all $i$.
\end{lemma}

\begin{proof}
Assume the linear dependence:
$$
  a\,\mathbf{v}_{\theta} + \sum_{i=1}^{d-1} b_i\,\mathbf{w}_i = \mathbf{0}.
$$
Taking the inner product with $\mathbf{v}_{\theta}$, we obtain:
$$
  a\,\langle\mathbf{v}_{\theta}, \mathbf{v}_{\theta}\rangle + \sum_{i=1}^{d-1} b_i\,\langle\mathbf{w}_i, \mathbf{v}_{\theta}\rangle = a\,\|\mathbf{v}_{\theta}\|^2 = a = 0,
$$
since $\|\mathbf{v}_{\theta}\| = 1$ and $\mathbf{v}_{\theta} \cdot \mathbf{w}_i = 0$. Next, taking the inner product with $\mathbf{w}_j$, we have:
$$
  \sum_{i=1}^{d-1} b_i\,\langle\mathbf{w}_i, \mathbf{w}_j\rangle = b_j,
$$
because $\langle\mathbf{w}_i, \mathbf{w}_j\rangle = 0$ for $i \neq j$ and $\langle\mathbf{w}_j , \mathbf{w}_j\rangle = 1$. Hence, $b_j = 0$. 

Since $j$ is arbitrary, we conclude that $b_i = 0$ for all $i$, and thus $a = 0$. This proves that the set $\{\mathbf{v}_{\theta}, \mathbf{w}_1, \dots, \mathbf{w}_{d-1}\}$ is linearly independent.
\end{proof}



Given the mapping $\mathbf{\hat{y}}(\psi, \mathbf{v}_{\theta}) = \tan(\psi)\,\mathbf{v}_{\theta}$, any infinitesimal change of $\mathbf{v}_{\theta}$ on the unit sphere $S^{d-1}$ lies in the tangent space $T_{\mathbf{v}_{\theta}}S^{d-1}$ (see Lemma~\ref{lemma_tangent_space}), ensuring $\mathbf{v}_{\theta}\cdot \mathrm{d}\mathbf{v}_{\theta}=0$. We introduce an orthonormal basis $\{\mathbf{w}_1,\dots,\mathbf{w}_{d-1}\}$ for this $(d-1)$-dimensional space. Consequently, each partial derivative with respect to $\mathbf{v}_{\theta}$ appears in the direction of some $\mathbf{w}_i$, yielding
$$
  \frac{\partial \mathbf{\hat{y}}}{\partial v_{\theta,i}} 
  \;=\; 
  \tan(\psi)\,\mathbf{w}_i,
  \quad
  i=1,\dots,d-1.
$$
Hence, $\mathbf{\hat{y}}$ increases linearly in each $\mathbf{w}_i$-direction, with a scaling factor $\tan(\psi)$.


Collecting \emph{all} these derivatives, The Jacobian matrix of the transformation is given by:
$$
\mathrm{D}\mathbf{\hat{y}}(\Phi, \mathbf{v}_{\theta}) =
\begin{bmatrix}
\sec^2(\psi)\,v_{\theta,1} & \tan(\psi)\,w_{1,1} & \dots & \tan(\psi)\,w_{d-1,1} \\
\sec^2(\psi)\,v_{\theta,2} & \tan(\psi)\,w_{1,2} & \dots & \tan(\psi)\,w_{d-1,2} \\
\vdots & \vdots & \ddots & \vdots \\
\sec^2(\psi)\,v_{\theta,d} & \tan(\psi)\,w_{1,d} & \dots & \tan(\psi)\,w_{d-1,d} \\
\end{bmatrix}.
$$
Here, the first column corresponds to $\frac{\partial \mathbf{x}}{\partial \Phi}$, and the subsequent $(d-1)$ columns represent $\frac{\partial \mathbf{x}}{\partial v_i}$ for each tangent direction $\mathbf{w}_i$.

Since $\{\mathbf{v}_{\theta}, \mathbf{w}_1, \dots, \mathbf{w}_{d-1}\}$ is an orthonormal set (as established in Lemma~\ref{lemma_linear_independence}, with $\langle\mathbf{v}_{\theta},\mathbf{w}_i\rangle=0$, $\|\mathbf{v}_{\theta}\|=1$, and $\|\mathbf{w}_i\|=1$), the determinant of $\mathrm{D}\mathbf{\hat{y}}(\psi,\mathbf{v}_{\theta})$ equals the product of the column norms:
$$
\det\bigl(\mathrm{D}\mathbf{\hat{y}}(\psi,\mathbf{v}_{\theta})\bigr)
\;=\;
\bigl\|\sec^2(\psi)\,\mathbf{v}_{\theta}\bigr\| 
\,\times\, 
\prod_{i=1}^{d-1}
\bigl\|\tan(\psi)\,\mathbf{w}_i\bigr\|.
$$
Noting that $\|\mathbf{v}_{\theta}\|=1$ and $\|\mathbf{w}_i\|=1$, the norms are
$$
\bigl\|\sec^2(\psi)\,\mathbf{v}_{\theta}\bigr\| \;=\; \sec^2(\psi),
\quad
\bigl\|\tan(\psi)\,\mathbf{w}_i\bigr\|
\;=\;
\tan(\psi).
$$
Hence, the determinant simplifies to
$$
\det\bigl(\mathrm{D}\mathbf{\hat{y}}(\Phi,\mathbf{v}_{\theta})\bigr) 
\;=\;
\bigl(\sec^2(\psi)\bigr)
\,\times\,
\bigl(\tan(\psi)\bigr)^{d-1}
\;=\;
\tan(\psi)^{\,d-1}\,\sec^2(\psi).
$$

\subsection{Measure Distortion Near $\psi\to \pi/2$}

\begin{theorem}[Regularization of Jacobian Divergence]
Let $\mathbf{\hat{y}}(\psi, \mathbf{v}_{\theta}) = \tan(\psi)\,\mathbf{v}_{\theta}$ be the homeomorphic mapping defined in HoP, where $\psi \in (0, \tfrac{\pi}{2})$ and $\mathbf{v}_{\theta} \in S^{d-1}$. The Jacobian determinant of $\mathbf{\hat{y}}$ is given by:
$$
\det\bigl(\mathrm{D}\mathbf{\hat{y}}(\psi, \mathbf{v}_{\theta})\bigr) = \tan(\psi)^{\,d-1}\,\sec^2(\psi).
$$
As $\psi \to \tfrac{\pi}{2}$, the Jacobian determinant diverges as:
$$
\det\bigl(\mathrm{D}\mathbf{\hat{y}}(\psi, \mathbf{v}_{\theta})\bigr) \sim \frac{1}{\epsilon^{\,d+1}},
$$
where $\epsilon = \tfrac{\pi}{2} - \psi$. To address this divergence, HoP introduces $\epsilon > 0$ as a regularization parameter, ensuring that:
\begin{enumerate}
\item The mapping $\mathbf{\hat{y}}$ remains well-defined and smooth for all $\psi \in (-\tfrac{\pi}{2}, \tfrac{\pi}{2} - \epsilon]$.
\item The Jacobian determinant is bounded by:
$$
\det\bigl(\mathrm{D}\mathbf{\hat{y}}(\psi, \mathbf{v}_{\theta})\bigr) \leq C \cdot \frac{1}{\epsilon^{\,d+1}},
$$
where $C$ is a constant depending only on $d$.
\end{enumerate}
This regularization prevents unbounded volume distortion and ensures numerical stability in optimization and sampling procedures.
\end{theorem}

\begin{proof}
The divergence follows from the asymptotic behavior of $\tan(\psi)$ and $\sec(\psi)$ as $\psi \to \tfrac{\pi}{2}$:
$$
\tan(\psi) \sim \frac{1}{\epsilon} \quad \text{and} \quad \sec^2(\psi) \sim \frac{1}{\epsilon^2}.
$$
Substituting these into the Jacobian determinant yields:
$$
\det\bigl(\mathrm{D}\mathbf{\hat{y}}(\psi, \mathbf{v}_{\theta})\bigr) = \tan(\psi)^{\,d-1}\,\sec^2(\psi) \sim \frac{1}{\epsilon^{\,d-1}} \cdot \frac{1}{\epsilon^2} = \frac{1}{\epsilon^{\,d+1}}.
$$
By restricting $\psi$ to $\psi \leq \tfrac{\pi}{2} - \epsilon$, the Jacobian determinant remains bounded, completing the proof.
\end{proof}

Thus, the Jacobian determinant diverges polynomially in $\tfrac{1}{\epsilon}$, reflecting a severe measure distortion in the limit $\psi\to \tfrac{\pi}{2}$. In practical terms, small parameter increments around $\psi \approx \tfrac{\pi}{2}$ map to disproportionately large volume elements in $\mathcal{Y}_\mathbf{x}$, leading to potential numerical instability if one attempts to sample or optimize directly over $\psi$ without truncation.

Although the Jacobian determinant $\det\bigl(\mathrm{D}\mathbf{\hat{y}}(\psi, \mathbf{v}_{\theta})\bigr)$ diverges as $\psi \to \tfrac{\pi}{2}$, practical considerations ensure numerical stability. Empirical knowledge or problem-specific constraints often imply an upper bound $\psi_{\text{max}} < \tfrac{\pi}{2}$, allowing the selection of $\epsilon = \tfrac{\pi}{2} - \psi_{\text{max}}$. Additionally, finite precision in computing hardware inherently limits the representable range, preventing true divergence in practice.

\section{Challenges of Stagnation in Polar Coordinate Optimization}\label{appendix:pco_analysis}
\begin{figure}[ht]
\begin{center}
\centerline{\includegraphics[width=0.9\textwidth]{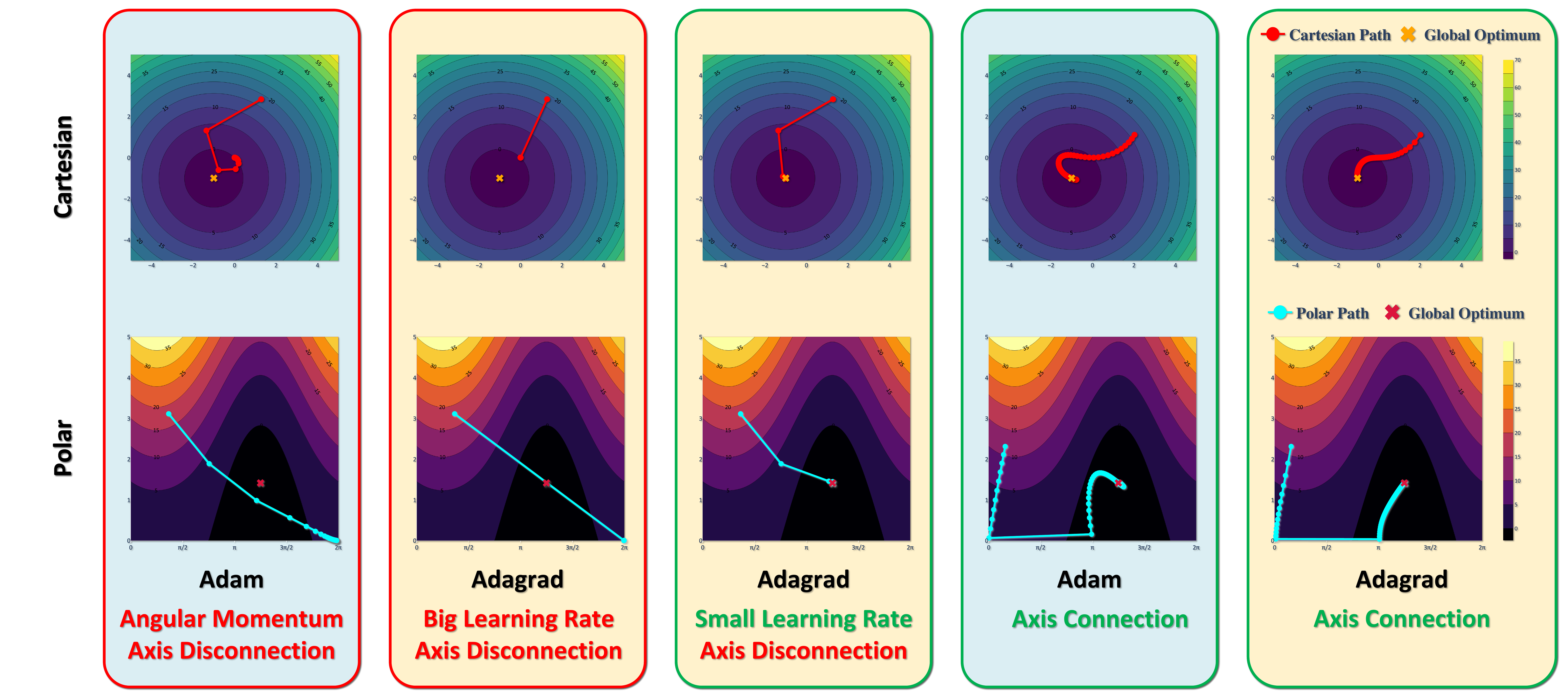}}
\caption{Illustration of optimization trajectories under polar coordinates. The figures demonstrate the impact of learning rate, momentum, and geometric reconnection on the convergence behavior.}
\label{polar-experiment-full}
\end{center}
\end{figure}

Optimization in polar coordinates introduces distinct challenges due to the coupling between radial and angular variables. As illustrated in Fig.~\ref{polar-experiment}, these challenges manifest as radial stagnation, angular freezing, and oscillatory behavior under large learning rates or momentum-based updates. Below, we formally state and address these issues.

\begin{remark} Radial Stagnation and Angular Freezing.\label{prop:radial_stagnation}
In polar coordinates, the non-negativity constraint $ r \geq 0 $ introduces radial stagnation and angular freezing when $ r = 0 $. Specifically, the radial update halts as:
\begin{flalign}
r_{t+1} = \max\{0, r_t - \eta (\cos\theta_t \frac{\partial f}{\partial x} + \sin\theta_t \frac{\partial f}{\partial y})\}.
\end{flalign}
When $ r_{t+1} = 0 $, angular updates are frozen:
\begin{flalign}
\theta_{t+1} = \theta_t + \eta r_{t+1} (\sin\theta_t \frac{\partial f}{\partial x} - \cos\theta_t \frac{\partial f}{\partial y}) = \theta_t.
\end{flalign}
This prevents the optimizer from escaping local regions and exploring global optima, particularly when the solution lies at $ \theta_t + \pi $.
\end{remark}
\begin{remark}Impact of Learning Rate and Momentum.\label{prop:LR_Momentum}
Large learning rates and momentum exacerbate radial stagnation and angular freezing. For a Lipschitz-continuous gradient $\|\nabla f\| \leq B$, radial truncation occurs when:
\begin{flalign}
\eta > \frac{r_t}{B}.
\end{flalign}
Momentum-based optimizers introduce oscillatory behavior near optima, as angular updates retain past gradient contributions, leading to overshooting or divergence.
\end{remark}
\begin{remark}Dynamic Learning Rate Strategy.\label{prop:Dynamic_LR}
To mitigate radial stagnation, dynamically scaling the learning rate as:
\begin{flalign}
\eta_t = \alpha r_t, \quad \text{where } \alpha < \frac{1}{B},
\end{flalign}
ensures $ r_{t+1} > 0 $. This prevents radial truncation by reducing step sizes near $ r = 0 $, while maintaining convergence stability. 
\end{remark}
Fig.~\ref{polar-experiment-full} (right) illustrates the effectiveness of geometric reconnection and adaptive learning strategies. The combined approach avoids radial stagnation, preserves gradient continuity, and facilitates efficient exploration of the optimization landscape.





\section{Experimental Data generation and Parameter Setup}\label{exp_setting_details}
In this section, the hyper-parameters of pervious experiments are provided. We split the section into two subsections to introduce the setting in Experiment \ref{exp:Synthetic} and Experiment \ref{miso_prob}, respectively. 

All experiments were conducted on a system with the following specifications:  

\begin{itemize}
    \item Host: Lambda Vector 1
    \item OS: Ubuntu 22.04.5 LTS x86\_64  
    \item CPU: AMD Ryzen Threadripper PRO 5955WX (32 cores) @ 4.000GHz  
    \item GPU: Single NVIDIA RTX A6000 
\end{itemize}

To simulate full parallelization, we report the total computation time divided by the number of test instances. Note that our implementations are not tightly optimized, and all timing comparisons should be interpreted as approximate.

\subsection{Synthetic Benchmarks \ref{exp:Synthetic}}

\subsubsection{(a) Polygon-Constrained Problem}


The objective function for this experiment is defined as in Eq. (\ref{QP_sin_Poly}). The problem is set in a 2-dimensional space, with the matrix $\mathbf{Q}$ randomly generated as a positive semi-definite matrix to ensure convex quadratic terms. Specifically, $\mathbf{Q}$ is created using the following procedure:
\begin{flalign}\label{semi-definite}
  \mathbf{Q} = \mathbf{A}^T \mathbf{A} + \mathbf{I},
\end{flalign}
where $\mathbf{A}$ is a randomly generated $2 \times 2$ matrix drawn from a standard Gaussian distribution, and $\mathbf{I}$ is the identity matrix to guarantee positive semi-definiteness.

The vector $\mathbf{b}$ is generated by sampling from a uniform distribution $\mathcal{U}(0, 2)$ and is checked against the polygon constraints using a validation function. Specifically, $\mathbf{b}$ is accepted only if it satisfies the closure conditions of the polygon constraints. Samples violating these constraints are discarded, ensuring all 20,000 instances in the dataset strictly satisfy the feasibility conditions.

The dataset is divided into a training set (70\%) and a test set (30\%). The parameters in the objective function are set as follows: the vector $\mathbf{p}$ is fixed to $[30, 30]$, and the sinusoidal term's parameters $\mathbf{q}$ are randomly generated for each instance, also drawn from $\mathcal{U}(0, 2)$. The constraint matrix $\mathbf{A}$ is precomputed and fixed for all experiments.

The center point $\mathbf{y}_0$ of the polygon is computed as the Chebyshev center of the feasible region. The Chebyshev center is obtained by solving a secondary optimization problem that maximizes the radius of the largest inscribed circle within the polygon constraints, ensuring that $\mathbf{y}_0$ lies strictly inside the feasible region.

\subsubsection{(b) $\ell_p$-norm Problem}


The objective function for this experiment is defined in Eq. (\ref{eq:lp-norm}). The matrix $\mathbf{Q}$ is randomly generated as a positive semi-definite matrix, ensuring convex quadratic terms, while the vector $\mathbf{p}$ is sampled from a standard normal distribution $\mathcal{N}(0, 1)$. The $\ell_p$-norm constraint uses $p = 0.5$, and the parameter $b$ is fixed to 1. The dataset consists of 20,000 samples, split into a training set ($70\%$) and a test set ($30\%$). The center point $\mathbf{y}_0$ is set to $[0, 0]$.

The positive semi-definite matrix $\mathbf{Q}$ is generated as described in Eq. (\ref{semi-definite}).

The dataset is constructed by generating feasible samples that satisfy the $\ell_p$-norm constraint. Instances violating the constraint are discarded, ensuring that the entire dataset strictly adheres to the defined feasibility conditions.

\subsubsection{High-Dimensional Semi-Unbounded Problem}

The objective function for high-dimensional semi-unbounded problem is given in Eq. (\ref{QP_sin_Poly}), where $\beta=30$, and $\mathbf{p}\sim\mathcal{N}$ is randomly sampled from the Gaussian Distribution $(-10,\mathbf{I})$. The parameters in the constraints are fixed, where $\mathbf{A}$ is randomly drawn from $\mathcal{N}(0,\mathbf{I})$ and $\mathbf{b}$ is drawn randomly as a positive number.  Since in this experiment constraints are fixed for a specific dataset, without lose of generality we define $\mathbf{b}=\mathbf{1}$. Moreover, the number of linear constraints is $d$, which is the dimension of $\mathbf{\hat{y}}$. Thus arbitrary semi-unbounded constraints are obtained while we have $\mathbf{0}$ as $\mathbf{y}_0$. The dataset is divided into a training set ($70\%$) and a test set ($30\%$). For training part, the learn rate is $10^{-4}$, the optimizer is Adam, and the training configuration includes $500$ epochs with a batch size of $2,048$. For $20$-dimensional problem, the dataset contains $20,000$ instance; 40, 60, 80 dimensions problems have $40,000$, $60,000$, $80,000$ instances, respectively. The MLP used for both all baseline methods and HoP has 512 neurons with ReLU activation functions.

\subsection{QoS-MISO WSR Problem}\label{QOSMISO_polar}
In this subsection the data preparation and how we apply HoP on QoS-MISO WSR problem are demonstrated. To generate the MISO simulation data, we apply algorithm given in Appendix \ref{FP_MISO_data_pre}. The user priority weight $\alpha_k\sim\mathcal{U}(0,1)$ with uniformization by $\alpha_k/(\sum_{k=1}^U \alpha_k)$. Channel state information is $\mathbf{h}_k$, followed circularly symmetric complex Gaussian (CSCG) distribution where the real part and imagine part follows $\mathcal{CN}(0,1)$. pathloss $=10$ and $\sigma^2=0.01$, $\delta_k\sim\mathcal{U}(0,1/3)$. $\text{P}_\text{max}=33$ dbm, $\text{P}_\text{c}=30$ dbm. For the training part, the learn rate is $10^{-2}$, the optimizer is Adam, and the training configuration includes $500$ epochs with a batch size of $64$. The dataset has $4000$ instances which is divided into a training set ($70\%$) and a test set ($30\%$).
Then, in the following part, we introduce how to formulate this multi-variables problem as a single variable problem which satisfies the NN's outputs format. The problem given in Eq. (\ref{WSR_pro}) can be reformulated as:
\begin{subequations}\label{miso_app}
\begin{flalign}
    \max_{\mathbf{w}_k} &\sum_{k=1}^U \alpha_k \log_{2}(1+\text{SINR}_k)\tag{\ref{miso_app}}\\
    \text{s.t.}\
    &\frac{\mathbf{w}_k^H\mathbf{h}_k\mathbf{h}_k^H\mathbf{w}_k}{\sum_{j\neq k}^U \mathbf{w}_j^H\mathbf{h}_k\mathbf{h}_k^H\mathbf{w}_j+\sigma_k^2}\geq \omega_k\label{sinr_cons}\\
    &\sum_{k=1}^U \text{P}_k \leq \text{P}_\text{max}-\text{P}_\text{c}
\end{flalign}
\end{subequations}
where $ \omega_k = 2^{\delta_k}-1$. Then, the constraint in Eq. (\ref{sinr_cons}) is identical to:
\begin{flalign}
    {\mathbf{w}_k^H\mathbf{h}_k\mathbf{h}_k^H\mathbf{w}_k}-\omega_k{\sum_{j\neq k}^U \mathbf{w}_j^H\mathbf{h}_k\mathbf{h}_k^H\mathbf{w}_j}\geq \omega_k\sigma_k^2\label{sinr_expand}
\end{flalign}
In this problem, for the convenience, $\mathbf{w}_k$ and $\mathbf{h}_k$ are reformulated as real vector and matrix by splicing:
\begin{flalign}\label{real_imag_reformH}
    &\Tilde{\mathbf{H}}_k = 
    \begin{bmatrix} 
    \Re{(\mathbf{h}_k\mathbf{h}_k^H)}&-\Im{(\mathbf{h}_k\mathbf{h}_k^H)}\\
    \Im{(\mathbf{h}_k\mathbf{h}_k^H)}&\Re{(\mathbf{h}_k\mathbf{h}_k^H)}
    \end{bmatrix} \\
    &\Tilde{\mathbf{w}}_k=\begin{bmatrix} 
    \Re{(\mathbf{w}_k)}\\\Im{(\mathbf{w}_k)}\end{bmatrix}\label{real_imag_reformw}
\end{flalign}
Therefore, Eq. (\ref{sinr_expand}) is represented by Eq. (\ref{real_imag_reformH}) and Eq. (\ref{real_imag_reformw}):
\begin{flalign}
    \Tilde{\mathbf{w}}_k^T\Tilde{\mathbf{H}}_k\Tilde{\mathbf{w}}_k - \omega_k\sum_{j\neq k}^U \Tilde{\mathbf{w}}_j^T\Tilde{\mathbf{H}}_k\Tilde{\mathbf{w}}_j\geq \omega_k\sigma_k^2\label{real_sinr}
\end{flalign}
Splicing all $\mathbf{\Tilde{w}}_k$ for each user, then Eq. (\ref{real_sinr}) is rewritten as:
\begin{flalign}
    \bar{\mathbf{w}}^T\bar{\mathbf{H}}_k\bar{\mathbf{w}} \geq \omega_k\sigma_k^2
\end{flalign}
where $\bar{\mathbf{w}}$, $\bar{\mathbf{H}}_k$ and $f_j(\omega_k)$ are defined as:
\begin{flalign}
    &\bar{\mathbf{w}}= [\Tilde{\mathbf{w}}_1^T,\Tilde{\mathbf{w}}_2^T,...,\Tilde{\mathbf{w}}_U^T]^T\\
    &\bar{\mathbf{H}}_k = \begin{bmatrix}
        f_1(\omega_k)\Tilde{\mathbf{H}}_k&\mathbf{0}&...&\mathbf{0}\\
        \mathbf{0}&f_2(\omega_k)\Tilde{\mathbf{H}}_k&...&\mathbf{0}\\
        \vdots&\vdots&\ddots&\vdots
        \\
        \mathbf{0}&\mathbf{0}&...&f_U(\omega_k)\Tilde{\mathbf{H}}_k
    \end{bmatrix}\\
    &f_j(\omega_k)
\begin{cases}
1, & {j = k}, \\
-\omega_k, & {j \neq k}.
\end{cases}
\end{flalign}
Based on above transformation, we apply the same operation on other constraints, then we have
\begin{flalign}
    \bar{\mathbf{w}}^T\bar{\mathbf{H}}_1\bar{\mathbf{w}} &\geq \omega_1\sigma_1^2\notag\\
    \bar{\mathbf{w}}^T\bar{\mathbf{H}}_2\bar{\mathbf{w}} &\geq \omega_2\sigma_2^2\notag\\
    &\vdots\notag\\
    \bar{\mathbf{w}}^T\bar{\mathbf{H}}_U\bar{\mathbf{w}} &\geq \omega_k\sigma_U^2\label{block_qos}\\
    \bar{\mathbf{w}}^T\mathbf{I}\bar{\mathbf{w}}&\leq \text{P}_\text{max}-\text{P}_\text{c}
\end{flalign}
Here, the original problem can be reorganized as an equivalent problem as follows:
\begin{subequations}\label{final_miso}
\begin{flalign}
    \max_{\mathbf{\bar{w}}_k} \sum_{k=1}^U \alpha_k \log_{2}&(1+\text{SINR}_k)\tag{\ref{final_miso}}\\
    \text{s.t.}\
    \bar{\mathbf{w}}^T\bar{\mathbf{H}}_1\bar{\mathbf{w}} &\geq \omega_1\sigma_1^2\notag\\
    \bar{\mathbf{w}}^T\bar{\mathbf{H}}_2\bar{\mathbf{w}} &\geq \omega_2\sigma_2^2\notag\\
    &\vdots\notag\\
    \bar{\mathbf{w}}^T\bar{\mathbf{H}}_U\bar{\mathbf{w}} &\geq \omega_U\sigma_U^2\\
    \bar{\mathbf{w}}^T\mathbf{I}\bar{\mathbf{w}}&\leq \text{P}_\text{max}-\text{P}_\text{c}
\end{flalign}
\end{subequations}
As consequence, $\bar{\mathbf{w}}$ is the estimated variable $\mathbf{\hat{y}}$ by HoP, where the NN's inputs, $\mathbf{x}$ are flatten $\Tilde{\mathbf{H}}_k$.
\section{FP for QoS-MISO WSR Experiment}\label{FP_MISO_data_pre}
We use the FP method to solve the original QoS-MISO WSR problem to compute the reference optimal beamformers results. According to \cite{shen2018fractional}, we reformulate the original problem as
\begin{subequations}\label{fp_prob}
\begin{flalign}
    \max_{\mathbf{W}_k} &\sum_{k=1}^U \alpha_k \log_{2}\big(1+2s_k\sqrt{\mathbf{h}_k^H\mathbf{W}_k\mathbf{h}_k}- s_k^2(\sigma_k^2+\sum_{j\neq k}^U(\mathbf{h}_k^H\mathbf{W}_j\mathbf{h}_k)\big) \tag{\ref{fp_prob}}\\
    \text{s.t.}\
    &\log_{2}(1+\frac{\mathbf{h}_k^H\mathbf{W}_k\mathbf{h}_k}{\sigma_k^2+\sum_{j\neq k}^U(\mathbf{h}_k^H\mathbf{W}_j\mathbf{h}_k)})\geq \delta_k\\
    &\sum_{k=1}^U \text{Tr}(\mathbf{W}_k) \leq \text{P}_\text{max}-\text{P}_\text{c}
\end{flalign}
\end{subequations}
where $s_k$ is auxiliary variable can be obtained by:
\begin{flalign}
    s_k = \frac{\sqrt{\mathbf{h}_k^H\mathbf{W}_k\mathbf{h}_k}}{\sigma_k^2+\sum_{j\neq k}^U\mathbf{h}_k^H\mathbf{W}_j\mathbf{h}_k}\label{z_k_update}
\end{flalign}
Note that the variables are applied with semidefinite relaxation such that $\mathbf{W}_k =\mathbf{w}_k\mathbf{w}_k^H $, $\text{rank}(\mathbf{W}_k) = 1$. Then, $\mathbf{w}_k$ is obtained by applying singular value decomposition (SVD) on $\mathbf{W}_k$. Therefore, the FP algorithm can solve QoS-MISO WSR problem by Algorithm \ref{alg:fp_qos_MISOWSR_data}.

\begin{algorithm}[tb]
   \caption{FP for QoS-MISO WSR optimization}
   \label{alg:fp_qos_MISOWSR_data}
\begin{algorithmic}
   \STATE {\bfseries Input:} Initialization parameters. Set counter $j=1 $ and convergence precision $\varphi_p$,
   \REPEAT
   \STATE Update $s_k$ by (\ref{z_k_update});
   \STATE Solve (\ref{fp_prob}) by cvxpy;
   \STATE Solve $\mathbf{w}_k$ by SVD;
   \UNTIL{$|\text{WSR}_{j+1}-\text{WSR}_{j}| \leq \varphi_p$}
\end{algorithmic}
\end{algorithm}


\end{document}